\documentclass[11pt,a4paper]{article}
\usepackage[hyperref]{emnlp2020}
\usepackage{times}
\usepackage{latexsym}

\usepackage{microtype}

\aclfinalcopy % Uncomment this line for the final submission
\usepackage{amsmath}
\usepackage{amssymb}
\usepackage{mathrsfs}
\usepackage{xcolor}
\usepackage{amsthm}
\usepackage{amsfonts}
\usepackage{graphicx}
\usepackage{stmaryrd}
\usepackage{algorithm}
\usepackage{xr}

\usepackage{times}
\usepackage{latexsym}
\usepackage{url}
\usepackage{array}
\usepackage{booktabs}
\usepackage{balance}
\usepackage{colortbl}
\usepackage{makecell}
\usepackage{soul}
\usepackage{blindtext}
\usepackage{multirow}
\usepackage{subcaption}
\usepackage{amssymb}% http://ctan.org/pkg/amssymb
\usepackage{pifont}% http://ctan.org/pkg/pifont
\newcommand{\cmark}{\ding{51}}%
\newcommand{\xmark}{\ding{55}}%

\newcolumntype{P}[1]{>{\centering\arraybackslash}p{#1}}
\usepackage{microtype}

\newcommand{\eqpunc}[1]{{\makebox[0pt][l]{\qquad\rm{#1}}}}

\theoremstyle{definition}
\newtheorem{definition}{Definition}[section]

\theoremstyle{remark}

\theoremstyle{plain}
\newtheorem{theorem}{Theorem}[section]

\theoremstyle{plain}

\theoremstyle{plain}

\newtheorem{proposition}[theorem]{Proposition}
\newtheorem{lemma}[theorem]{Lemma}

\usepackage{graphicx,wrapfig,lipsum}

\usepackage{amsmath,amsfonts,bm}

\def\eqref#1{(\ref{#1})}
\def\1{\bm{1}}

\def\rmI{{\mathbf{I}}}

\def\vzero{{\bm{0}}}

\def\va{{\bm{a}}}

\def\vc{{\bm{c}}}

\def\vq{{\bm{q}}}

\def\vv{{\bm{v}}}

\def\vx{{\bm{x}}}
\def\vy{{\bm{y}}}
\def\vz{{\bm{z}}}

\def\mX{{\bm{X}}}

\DeclareMathAlphabet{\mathsfit}{\encodingdefault}{\sfdefault}{m}{sl}
\SetMathAlphabet{\mathsfit}{bold}{\encodingdefault}{\sfdefault}{bx}{n}

\usepackage{commath}
\usepackage{amssymb}
\usepackage{amsmath,amsfonts,bm}

\newcommand{\br}{\mathbb{R}}

\newcommand{\dyck}{Dyck-1}
\newcommand{\shuff}{Shuffle-Dyck}
\newcommand{\bool}{BoolExp}
\newcommand{\booln}{BoolExp-$n$}
\newcommand{\kshuff}{Shuffle-$k$}
\newcommand{\cnt}[1]{\ensuremath\sigma(#1)}
\newcommand{\Ctwo}{\ensuremath a^nb^n}
\newcommand{\Cthree}{\ensuremath a^nb^nc^n}
\newcommand{\dn}[1]{\ensuremath \mathcal{D}_{#1}}

\newcommand{\ResetDyck}{Reset-Dyck-$1$}

\newcommand{\Att}{\operatorname{Att}}
\newcommand{\<}{\ensuremath \langle}
\renewcommand{\>}{\ensuremath \rangle}

\newcommand{\softmax}{\ensuremath{\mathsf{softmax}}}

\title{On the Ability and Limitations of Transformers to \\ Recognize Formal Languages}

\author{Satwik Bhattamishra$^\spadesuit$  \quad Kabir Ahuja$^\diamondsuit$\thanks{\, This research was conducted during the author's internship at Microsoft Research.} \quad Navin Goyal$^\spadesuit$\\
	$^\spadesuit$ Microsoft Research India\\
	$^\diamondsuit$ Udaan.com\\
	{\tt \small \{t-satbh,navingo\}@microsoft.com} \\
	{\tt \small kabir.ahuja@udaan.com} \\
}

\date{}

\begin{document}
\maketitle

\begin{abstract}
Transformers have supplanted recurrent models in a large number of NLP tasks. However, the differences in their abilities to model different syntactic properties remain largely unknown. Past works suggest that LSTMs generalize very well on regular languages and have close connections with counter languages. In this work, we systematically study the ability of Transformers to model such languages as well as the role of its individual components in doing so. We first provide a construction of Transformers for a subclass of counter languages, including well-studied languages such as $n$-ary Boolean Expressions, Dyck-1, and its generalizations. In experiments, we find that Transformers do well on this subclass, and their learned mechanism strongly correlates with our construction. Perhaps surprisingly, in contrast to LSTMs, Transformers do well only on a subset of regular languages with degrading performance as we make languages more complex according to a well-known measure of complexity. Our analysis also provides insights on the role of self-attention mechanism in modeling certain behaviors and the influence of positional encoding schemes on the learning and generalization abilities of the model. 
\end{abstract}

\section{Introduction}
%Transformer \cite{vaswani2017attention} is a self-attention based architecture which has led to state of the art results across various NLP tasks including machine translation \cite{vaswani2017attention}, language modeling \cite{radford2018improving}, question answering \cite{devlin-etal-2019-bert,liu2019roberta} and also beyond NLP \cite{Trade-Smola, Choppy-Tomkins}. A common theme in many of these works has been that self-attention-based models improve performance of the recurrent models. Much effort has been devoted to understand the intermediate representations of pretrained Transformer models \cite{tenney-etal-2019-bert,coenen2019visualizing, ethayarajh2019contextual}. Recently, Transformers have been shown to be Turing-complete when provided with unbounded precision \cite{perez2019turing}. However, our understanding of their practical ability to model different behaviors relevant for sequence modeling tasks is still nascent.

%Transformer \cite{vaswani2017attention} is a self-attention based architecture which has led to state-of-the-art results across various NLP tasks \cite{devlin-etal-2019-bert,liu2019roberta,radford2018improving}

 Transformer \cite{vaswani2017attention} is a self-attention based architecture which has led to state-of-the-art results across various NLP tasks \cite{devlin-etal-2019-bert,liu2019roberta,radford2018improving}. Much effort has been devoted to understand the inner workings and intermediate representations of pre-trained models; \citet{rogers2020primer} is a recent survey. However, our understanding of their practical ability to model different behaviors relevant to sequence modeling is still nascent.

On the other hand, a long line of research has sought to understand the capabilities of recurrent neural models such as the LSTMs \cite{hochreiter1997long} . Recently, \citet{weiss2018practical}, \citet{suzgun2019lstm} showed that LSTMs are capable of recognizing counter languages such as \dyck{} and $\Ctwo$ by learning to perform counting like behavior. \citet{suzgun2019lstm} showed that LSTMs can recognize shuffles of multiple \dyck{} languages, also known as \shuff{}. Since Transformer based models (e.g., GPT-2 and BERT) are not equipped with recurrence and start computation from scratch at each step, they are incapable of directly maintaining a counter. 
Moreover, it is known that theoretically RNNs can recognize any regular language in finite precision, and LSTMs work well for this task in practical settings. However, Transformer's ability to model such properties in practical settings remains an open question.

\begin{figure}[t]
 \centering
	\includegraphics[scale=0.36]{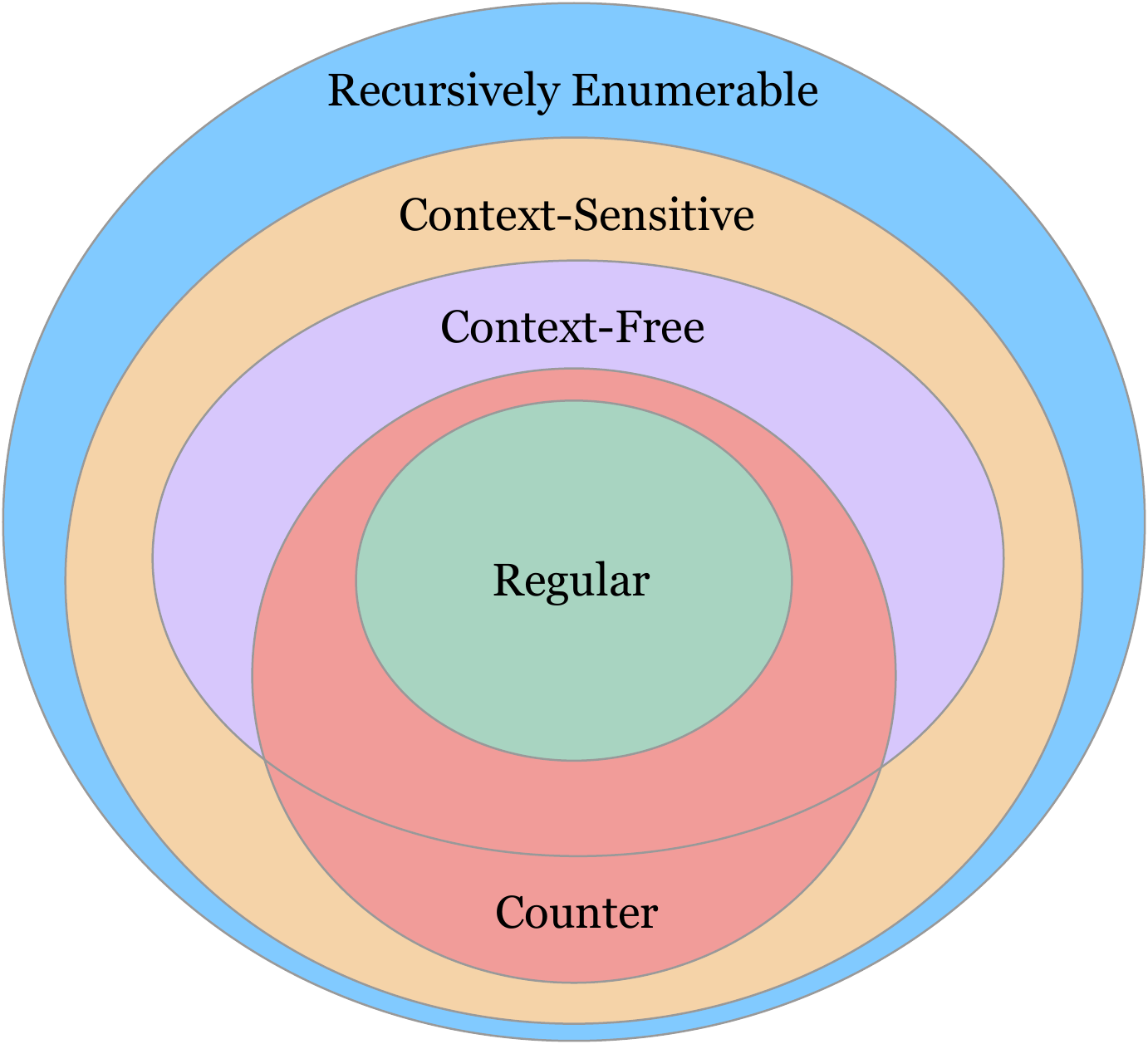}
	\caption{\label{fig:intro} Counter languages form a strict superset of regular languages, and are a strict subset of context-sensitive languages. Counter and context-free languages have a nonempty intersection and neither set is contained in the other.}
\end{figure}

%In this work, we investigate the ability of Transformers to express, learn and generalize on certain counter and regular languages.

% We take a step towards characterizing the class of language that are recognizable by Transformers.

%  
%whose properties are well understood with respect to LSTMs

Prior to the current dominance of Transformers for NLP tasks, recurrent models like RNN-based models such as LSTMs were the most common choice, and their computational capabilities have been studied for decades, e.g., \cite{kolen2001field}.  In this work, we investigate the ability of Transformers to express, learn, and generalize on certain counter and regular languages. Formal languages provide us a controlled setting to study a network's ability to model different syntactic properties in isolation and the role of its individual components in doing so. 
%The primary reason for our choice of counter languages is that 

Recent work has demonstrated close connections between LSTMs and counter automata. Hence, we seek to understand the capabilities of Transformers to model languages for which the abilities of LSTMs are well understood. We first show that Transformers are expressive enough to recognize certain counter languages like Shuffle-Dyck and $n$-ary Boolean Expressions by using self-attention mechanism to implement the relevant counter operations in an indirect manner. We then extensively evaluate the model's learning and generalization abilities on such counter languages and find that models generalize well on such languages. Visualizing the intermediate representations of these models shows strong correlations with our proposed construction. Although Transformers can generalize well on some popularly used counter languages, we observe that they are limited in their ability to recognize others. We find a clear contrast between the performance of Transformers and LSTMs on regular languages (a subclass of counter languages). Our results indicate that, in contrast to LSTMs, Transformers achieve limited performance on languages that involve modeling periodicity, modular counting, and even simpler star-free variants of \dyck{}, which they were able to recognize effortlessly. Our analysis provides insights about the significance of different components, namely self-attention, positional encoding, and the number of layers. Our results also show that positional masking and positional encoding can both aid in generalization and training, but in different ways. We conduct extensive experiments on over 25 carefully chosen formal languages. Our results are perhaps the first indication of the limitations of Transformers for practical-sized problems that are, in a precise sense, very simple, and in particular, easy for recurrent models.

\section{Related Work}\label{sec:relwork}

Numerous works, e.g., \citet{suzgun2018evaluating,sennhauser-berwick-2018-evaluating,skachkova-etal-2018-closing}, have attempted to understand the capabilities and inner workings of recurrent models by empirically analyzing them on formal languages. \citet{weiss2018practical} showed that LSTMs are capable of simulating counter operations and explored their practical ability to recognize languages like $\Ctwo$ and $\Cthree$. \citet{suzgun2019lstm} further showed that LSTMs can learn to recognize \dyck{} and \shuff{} and can simulate the behavior of $k$-counter machines. Theoretical connections of recurrent models have been established with counter languages \cite{merrill2019sequential, merrill2020formal, merrill2020linguistic}. It has also been shown that RNN based models can recognize regular languages \cite{kolen2001field,korsky2019computational} and efforts have been made to extract DFAs from RNNs trained to recognize regular languages \cite{weiss2019learning, wang2018comparative, michalenko2019representing}. We are not aware of such studies for Transformers.

Recently, researchers have sought to empirically analyze different aspects of the Transformer trained on practical NLP tasks such as the information contained in intermediate layers \cite{rogers2020primer,reif2019visualizing,warstadt-etal-2019-investigating}.  \citet{voita-etal-2019-analyzing} studied the role of different types of attention heads. \citet{yang2019assessing, tsai2019transformer}  examined the ability of the model to learn order information via different positional encoding schemes. Complementary to these, our work is focused on analyzing Transformer's ability to model particular behaviors that could be relevant to modeling linguistic structure. Recently, it has been shown that Transformers are Turing-complete \cite{perez2019turing,bhattamishra2020computational} and are universal approximators of sequence-to-sequence functions given arbitrary precision \cite{yun2019transformers}. \citet{Hahn} shows that Transformers cannot recognize languages Parity and Dyck-2. However, these results only apply to very long words, and their applicability to practical-sized inputs is not clear (indeed, we will see different behavior for practical-sized input). Moreover, these results concern the expressive power of Transformers and do not apply to learning and generalization abilities. Thus Transformers' ability to model formal languages requires further investigation.

%It has also been shown that RNN based models can recognize regular languages \cite{kolen2001field,korsky2019computational} and efforts have been made to extract DFAs from RNNs trained to recognize regular languages \cite{weiss2019learning, wang2018comparative, michalenko2019representing}. We are not aware of such studies for Transformers.

% \citet{voita-etal-2019-analyzing} studied the role of different types of attention heads. \citet{yang2019assessing, tsai2019transformer}  examined the ability of the model to learn order information via different positional encoding schemes. \citet{lin-etal-2019-open}, \citet{tenney-etal-2019-bert} have shown that transformer-based models are capable of encoding some form of syntactic information in its intermediate layer. \cite{hu2020systematic} test the syntactic abilities of Transformer based language models on a suite of tasks.
%\footnote{In some cases we found them to be exactly uniform and in some cases they were very close to uniform depending on the random seed used to initialize the parameters.}

\section{Definitions}\label{sec:def}
We consider the Transformer as used in popular pre-trained LM models such as BERT and GPT, which is the encoder-only model of the original seq-to-seq architecture \cite{vaswani2017attention}. The encoder consists of multiple layers with two blocks each: (1) self-attention block, (2) a feed-forward network (FFN). For $1 \leq i \leq n$, at the $i$-th step, the model takes as input the sequence $s_1, s_2, \ldots, s_i$ where $s \in \Sigma$ and generates the output vector $\vy_i$. 
%We use $f_e$ to denote an embedding function that takes as input a symbol $s \in \Sigma$ and produces a vector $\tilde{\vx} \in \br^{d_{model}}$.
Each input $s_i$ is first converted into an embedding vector using the function $f_e: \Sigma \to \br^{d_{\mathrm{model}}}$ and usually some form of positional encoding is added to yield the final input vector $\vx_i$.  The embedding dimension $d_{\mathrm{model}}$ is also the dimension of intermediate vectors of the network.  Let $\mX_i := (\vx_1, \ldots, \vx_i)$ for $i \geq 1$.

 In the self-attention block, the input vectors undergo linear transformations $Q(\cdot), K(\cdot),\text{ and } V(\cdot)$ yielding the corresponding query, key and value vectors, respectively.
 The self-attention mechanism takes as input a \emph{query} vector
 $Q(\vx_i)$, \emph{key} vectors $K(\mX_i)$, and \emph{value} vectors $V(\mX_i)$. %All vectors are in $\bq^d$. 	
 An attention-head denoted by $\Att(Q(\vx_i), K(\mX_i), V(\mX_i))$, is a vector $\va_i = \sum_{j=1}^{i}\alpha_j \vv_j$, where $(\alpha_1, \ldots , \alpha_i) = \softmax(\<Q(\vx_i), K(\vx_1) \>, \ldots , \< Q(\vx_i), K(\vx_i)\>)$.

The output of a layer denoted by $\vz_i$ is computed by $\vz_i = O(\va_i)$ where $1 \leq i \leq n$ and $O(\cdot)$ typically denotes an FFN with $\mathsf{ReLU}$ activation. The complete $L$-layer model is a repeated application of the single-layer model described above, which produces a vector $\vz_i^{L}$ at its final layer where $L$ denotes the last layer. The final output is obtained by applying a projection layer with some normalization or an FFN over the vectors $\vz_i^{L}$'s and is denoted by $\vy_i=F(\vz_i^{L})$.  Residual connections and layer normalization are also applied to aid the learning process of the network.

In an LM setting, when the Transformer processes the input sequentially, each input symbol can only attend over itself and the previous inputs, masking is applied over the inputs following it. Note that, providing positional information in this form via masked self-attention is also referred to as positional masking \cite{vaswani2017attention,Shen2018DiSANDS}. A Transformer model without positional encoding and positional masking is order-insensitive. 

\subsection{Formal Languages}\label{sec:formaldef}
%Study of formal language theory was initiated by Chomsky (1957) to model natural languages. Much of the complexity of natural languages is abstracted away in this formalization including the meaning and frequency of words.  Yet, this theory has been remarkably effective and influential not only in linguistics but 
%in computer science (e.g. programming languages) and beyond. 
Formal languages are abstract models of the syntax of programming and natural languages; they also relate to cognitive linguistics, e.g., \citet{Jaeger-Rogers, Hahn} and references therein. 

%\texttt{Introduce and define the languages we evaluate on and the differences between them. Refer \cite{hu2020systematic}}. Space Approx 1 column
%   Although languages such as \dyck{} and $\Ctwo$ are context-free Languages, a DCA with a single counter is sufficient to recognize them. Similarly, a DCA with two single-turn counters is sufficient to recognize $\Cthree$. In our work, we explore the Transformer's ability to model such languages and consider some general form of counter languages such as \shuff{} and n-ary boolean expressions. Shuffled-Dycks contain shuffles of multiple Dyck-1 languages. Recognizing \shuff{} requires a DCA equipped with multiple multi-turn counters where for a given type of bracket, its corresponding counter is incremented or decremented by 1. Hence it represents a more general form of counter language. Similarly, recognizing n-ary expressions requires a $1$-counter automaton with different counter updates depending on the operator. A ternary operator will increment the counter by $3$ whereas a unary operator will increment it by $1$.
% In our study, we consider several languages belonging to different parts of the Chomsky hierarchy where each type of language requires modelling a certain behavior. Here we describe some of the key languages on which analyze on and their respective properties.
\noindent \textbf{Counter Languages.} These are languages recognized by a deterministic counter automaton (DCA), that is, a DFA with a finite number of unbounded counters \cite{fischer1968counter}. The counters can be incremented/decremented by constant values and can be reset to $0$ 
(details in App. \ref{subsec:counter_automata}). The commonly used counter languages to study sequence models are \dyck{}, $\Ctwo$, and $\Cthree$.  Several works have explored the ability of recurrent models to recognize these languages as well as their underlying mechanism to do so. We include them in our analysis as well as some general form of counter languages such as \shuff{} (as used in \citet{suzgun2019lstm}) and $n$-ary Boolean Expressions. The language Dyck-1 over alphabet $\Sigma =\{[,]\}$ consists of balanced parentheses defined by derivation rules $S \rightarrow [ \; S \; ]\; |\; SS \;| \; \epsilon$. \shuff{} is a family of languages containing shuffles of Dyck-1 languages. Shuffle-$k$ denotes the shuffle of $k$ Dyck-1 languages: 
it contains $k$ different types of brackets, where each type of bracket is required to be well-balanced, but their relative order is unconstrained. For instance, a Shuffle-2 language over alphabet $\Sigma= \{[, ], (,  )\}$ contains the words $([)]$ and $[((]))$ but not $])[($. We also consider $n$-ary Boolean Expressions (hereby \bool{}-$n$), which are a family of languages of valid Boolean expressions (in the prefix notation) parameterized by the number of operators and their individual arities. For instance, an expression with unary operator $\sim$ and binary operator $\land$ contains the word `$\land \sim01$' but not `$\sim 10$' (formal definitions in App. \ref{sec:definitions}).

Note that, although languages such as \dyck{} and $\Ctwo$ are context-free, a DCA with a single counter is sufficient to recognize \dyck{} and $\Ctwo$. Similarly, a DCA with two single-turn counters can recognize $\Cthree$. On the other hand, recognizing \shuff{} requires multiple multi-turn counters, where for a given type of bracket, its corresponding counter is incremented or decremented by 1. Hence, it represents a more general form of counter languages. Similarly, recognizing 
 %$n$-ary Boolean Expressions
 \bool{} requires a $1$-counter DCA with the counter updates depending on the operator: a ternary operator will increment the counter by $2$ ($=\mathrm{arity}-1$) whereas a unary operator will increment it by $0$. Figure \ref{fig:intro} shows the relationship between counter languages and other classes of formal languages.

 \noindent \textbf{Regular Languages.} Regular languages, perhaps the best studied class of formal languages, form a subclass of counter languages\footnote{For simplicity, from now on, we will refer to a particular language as a counter language if a DCA with a nonzero number of counters is necessary to recognize it, else we will refer to it as a regular language.}. They neatly divide into two subclasses: \textit{star-free} and 
 \textit{non-star-free}. Star-free languages can be described by regular expressions formed by union, intersection, complementation, and concatenation operators but not the Kleene star ($*$). %By non-star-free languages, we will mean regular languages that are not star-free. 
 Like regular languages, star-free languages are surprisingly rich with algebraic, logical, and multiple other characterizations and continue to be actively researched, e.g., \cite{McNaughton-Papert, Jaeger-Rogers}. They form a simpler subclass of regular languages where the notion of simplicity can be made precise in various ways, e.g. they are first-order logic definable and cannot represent languages that require modular counting.

 We first consider Tomita grammars containing 7 regular languages representable by DFAs of small sizes, a popular benchmark for evaluating recurrent models and extracting DFA from trained recurrent models (see, e.g., \citet{Comparative-Giles}). Tomita grammars contain both star-free and non-star-free languages.  We further investigate some non-star-free languages such as $(aa)^*$, Parity and $(abab)^*$. Parity contains words over $\{0,1\}$ with an even number of $1$'s. Similarly $(aa)^*$ and $(abab)^*$ require modeling periodicity. 
 
 On the other hand, the seemingly similar looking language $(ab)^*$ is star-free: $(ab)^* = (b\emptyset^c + \emptyset^c a + \emptyset^c aa \emptyset^c + \emptyset^c bb \emptyset^c)^c$, where $\cdot^c$ denotes set complementation, and thus $\emptyset^c = \Sigma^*$. The dot-depth of a star-free language is a measure of nested concatenation or sequentiality required in a star-free regular expression 
 (formal definition in App. \ref{subsec:dot-depth}). We define a family $\dn{0}, \dn{1}, \ldots  $ of star-free languages. For $n \geq 0$, the language $\dn{n}$ over $\Sigma=\{a,b\}$ is defined inductively as follows: $\dn{n} = (a\dn{n-1}b)^*$ where $\dn{0} = \epsilon$, the empty word. Thus $\dn{1}= (ab)^*$ and $\dn{2} = (a(ab)^*b)^*$. Language $\dn{n}$ is known to have dot-depth $n$.
 
 The list of all considered languages and their definitions are provided in the App. \ref{sec:definitions}.

\section{Expressiveness Results}\label{sec:const}

\begin{proposition}
	There exists a Transformer as defined in Section \ref{sec:def} that can recognize the family of languages Shuffle-Dyck.
\end{proposition}

\begin{proof}

	Let $s_1, s_2, \ldots, s_n$ denote a sequence $w \in $ \kshuff{} over the alphabet $\Sigma=\{[_0, \ldots, [_{k-1}, ]_0, \ldots,]_{k-1}\}$. The language Shuffle-$1$ is equivalent to \dyck{}. For any \kshuff{} language, consider a model with $d_{model}= 2k$, where the embedding function $f_e$ is defined as follows. For each type of open bracket $[_j$ where,$0 \leq j < k$, the vector $f_e([_j) $ has the value $+1$ and $-1$ at the indices $2j$ and $2j+1$, respectively. It has the value $0$ at the rest of the indices. Similarly for each closing bracket, the vector $f_e(]_j) $ has the value $-1$ and $+1 $ at the indices $2j$ and $2j+1$, and it has the value $0$ at the rest of the indices. For \dyck{}, this would lead to $f_e([ ) = [+1, -1]^T$ and $f_e(]) = [-1,+1]^T$ (with $d_{model}=2$). We use a single-layer Transformer where we set the matrix corresponding to linear transformation for key vectors to be null matrix, that is $K(\vx) = \vzero$ for all $\vx$. This will lead to equal attention weights for all inputs. The matrices corresponding to $Q(\cdot)$ and $V(\cdot)$ are set to Identity. Thus, $\Att(Q(\vx_i), K(\mX_i), V(\mX_i)) = \frac{1}{i}\sum_{j=1}^{i}\vv_j$	for $1 \leq i \leq n$. Hence, at the $i$-th step, the self-attention block produces a vector $\va_i$ which has the values $\frac{\cnt{[_j}-\cnt{]_j}}{i}$ at indices $2j$ and the values $\frac{\cnt{]_j}-\cnt{[_j}}{i}$ at indices $2j+1$, where $\cnt{s}$ denotes the number of occurrence of the symbol $s$. For instance, in \dyck{}, if in the first $i$ inputs, there are $\cnt{[}$ open brackets and $\cnt{]}$ closing brackets, then $\va_i = [\frac{\cnt{[}-\cnt{]}}{i}, \frac{\cnt{]}-\cnt{[}}{i}]^T$, where $i=\cnt{[}+\cnt{]}$. In $\va_i$, the value $\cnt{[}-\cnt{]}$ represents the depth 
	(difference between the number of open and closing brackets) of the \dyck{} word at index $i$. Hence, the first coordinate is the ratio of the depth of the \dyck{} word and its length at that index, while the other coordinate is its negative.
	
    We then apply a simple FFN with $\mathsf{ReLU}$ activation over the vector $\va_i$. The vector $\vz_i = \mathsf{ReLU}(\rmI\va_i)$. The even indices of the vector $\vz_i$ will be nonzero if the number of open brackets of the corresponding type is greater than the number of closing brackets. A similar statement holds for the odd indices. Thus, for a given word to be in \kshuff{}, the values at odd indices of the vector $\vz_i$ must never be nonzero, and the values of all coordinates must be zero at the last step to ensure the number of open and closing brackets are the same.

For an input sequence  $s_1, s_2, \ldots, s_n$, the model will produce $\vz_1, \ldots ,\vz_n$ based on the construction specified above. A word $w$ belongs to language \kshuff{} if $\vz_{i,2j+1} = 0$ for all $1 \leq i \leq n$, $0 \leq j <k$ and $\vz_{n} =\vzero$ and does not belong to the language otherwise. This can be easily implemented by an additional layer of self-attention and feedforward network to classify a given sequence.

\end{proof}

\begin{table*}[t]
	\scriptsize{\centering
		\begin{tabular}{P{12em}p{20em}P{7em}P{7em}P{7em}}
			\toprule
			\textbf{Language} & \textbf{Model} & \textbf{Bin-1 Accuracy [1, 50]}$\uparrow$ & \textbf{Bin-2 Accuracy [51, 100]}$\uparrow$ & \textbf{Bin-3 Accuracy [101, 150]}$\uparrow$ \\
			\midrule
			\multirow{5}{*}{\textbf{Shuffle-$2$}} & \textbf{LSTM (Baseline)} & \textbf{100.0}  & \textbf{100.0} & \textbf{100.0}\\
			\cmidrule{2-5}
			&\textbf{Transformer (Absolute Positional Encodings)} & \textbf{100.0} & 85.2 & 63.3 \\
			&\textbf{Transformer (Relative Positional Encodings)} & \textbf{100.0} & 51.6 & 3.8 \\
			&\textbf{Transformer (Only Positional Masking)} & \textbf{100.0} & \textbf{100.0} & 93.0 \\
			\midrule
			%\toprule
			\multirow{5}{*}{\textbf{\bool{}-$3$}} & \textbf{LSTM (Baseline)} & \textbf{100.0}  & \textbf{100.0} & 99.7 \\
			\cmidrule{2-5}
			&\textbf{Transformer (Absolute Positional Encodings)} & \textbf{100.0} & 90.6 & 51.3 \\
			&\textbf{Transformer (Relative Positional Encodings)} & \textbf{100.0} & 96.0 & 68.4 \\
			&\textbf{Transformer (Only Positional Masking)} & \textbf{100.0} & \textbf{100.0} & \textbf{99.8} \\
			\midrule
			%\toprule
			\multirow{5}{*}{\textbf{$\Cthree$}} & \textbf{LSTM (Baseline)} & \textbf{100.0}  & \textbf{100.0} & 97.8 \\
			\cmidrule{2-5}
			&\textbf{Transformer (Absolute Positional Encodings)} & \textbf{100.0} & 62.1 & 5.3 \\
			&\textbf{Transformer (Relative Positional Encodings)} & \textbf{100.0} & 31.3 & 22.0 \\
			&\textbf{Transformer (Only Positional Masking)} & \textbf{100.0} & \textbf{100.0} & \textbf{100.0} \\
			\midrule
		\end{tabular}
		\caption{\label{tab:results} The performance of Transformers and LSTMs on the respective counter languages. Refer to section \ref{sec:res_count} for details. Performance on other counter languages such as Shuffle-4 and Shuffle-6 are listed in Table \ref{tab:countresults} in appendix. }
	}
\end{table*}

The bottleneck for precision in the construction above is the calculation of values of the form $\frac{\cnt{[}-\cnt{]}}{i}$ in the vector $\va_i$. Since in a finite precision setting with $r$ bits, this can be computed up to a value exponential in $r$, our proof entails that Transformers can recognize languages in \shuff{} for lengths exponential in the number of bits. % in such a setting.

Using a similar logic, one can also show that Transformers can recognize the family of languages \bool{}-$n$ (refer to Lemma \ref{lem:bool}). By setting the value vectors according to the arities of the operators, the model can obtain the ratio of the counter value of the underlying automata and the length of the input at each step via self-attention. Although the above construction is specific to these language families, we provide a proof for a more general but restricted subclass of Counter Languages in the appendix (refer to Lemma \ref{lem:qscm}). The above construction serves to illustrate how Transformers can recognize such languages by indirectly doing relevant computations. As we will later see, this will also help us interpret how trained models recognize such languages.

\section{Experimental Setup}\label{sec:exp}

In our experiments, we consider $27$ formal languages belonging to different parts in the hierarchy of counter and regular languages. For each language, we generate samples within a fixed-length window for our training set and generate multiple validation sets with different windows of length to evaluate the model's generalization ability.

For most of the languages, we generate 10k samples for our training sets within lengths 1 to 50 and create different validation sets containing samples with distinct but contiguous windows of length. The number of samples in each validation set is 2k, and the width of each window is about 50. For languages that have very few positive examples in a given window of length, such as $(ab)^*$ and $\Cthree$, we train on all positive examples within the training window. Similarly, each validation set contains all possible strings of the language for a particular range. Table \ref{tab:stats} in appendix lists the dataset statistics of all $27$ formal languages we consider.\footnote{Our experimental setup closely follows the setup of \citet{suzgun2019lstm,suzgun2018evaluating} for RNNs}. We have made our source code available at \href{https://github.com/satwik77/Transformer-Formal-Languages}{https://github.com/satwik77/Transformer-Formal-Languages}.

\subsection{Training details}\label{sec:train_details}

We train the model on character prediction task as introduced in  \citet{gers2001lstm} and as used in \citet{suzgun2018evaluating,suzgun2019lstm}. Similar to an LM setup, the model is only presented with positive samples from the given language. For an input sequence $s_1, s_2, \ldots, s_n$, the model receives the sequence $s_1, \ldots, s_i$ for $1 \leq i \leq n$ at each step $i$ and the goal of the model is to predict the next set of legal/valid characters in the $(i+1)^{th}$ step. From here onwards, we say a model can recognize a language if it can perform the character prediction task perfectly.

The model assigns a probability to each character in the vocabulary of the language corresponding to its validity in the next time-step. The output can be represented by a $k$-hot vector where each coordinate corresponds to a character in the vocabulary of the language. The output is computed by applying a sigmoid activation over the scores assigned by the model for each character. Following \citet{suzgun2018evaluating,suzgun2019lstm}, the learning objective of the model is to minimize the mean-squared error between the predicted probabilities and the $k$-hot labels.\footnote{We also tried BCE loss in our initial experiments and found similar results for languages such as Parity, Tomita grammars and certain counter languages.} During inference, we use a threshold of 0.5 to obtain the predictions of the model. For a test sample, the model's prediction is considered to be correct if and only if its output at every step is correct. Note that, this is a relatively stringent metric as a correct prediction is obtained only when the output is correct at every step. The accuracy of the model over test samples is the fraction of total samples predicted correctly\footnote{A discussion on the choice of character prediction task and its relation to other tasks such as standard classification and LM is provided in section \ref{sec:cpr} in the appendix.}. Similar to \citet{suzgun2019lstm} we consider models of small sizes to prevent them from memorizing the training set and make it feasible to visualize the model. In our experiments, we consider Transformers with up to 4 layers, 4 heads and the dimension of the intermediate vectors within 2 to 32. We extensively tune the model across various hyperparameter settings. We also examine the influence of providing positional information in different ways such as absolute encodings, relative encodings \cite{dai-etal-2019-transformer} and using only positional masking without any explicit encodings.

\section{Results on Counter Languages}\label{sec:res_count}

We evaluated the performance of the model on $9$ counter languages. Table \ref{tab:results} shows the performance of different models described above on some representative languages. We also include the performance of  LSTMs as a baseline. We found that Transformers of small size (single head and single layer) can generalize well on some general form of counter languages such as \shuff{} and \booln{}. Surprisingly, we observed this behavior when the network was not provided any form of explicit positional encodings, and positional information was only available in the form of masking. For models with positional encoding, the lack of the ability to generalize to higher lengths could be attributed to the fact that the model has never been trained on some of the positional encodings that it receives at test time. On the other hand, the model without any explicit form of positional encoding is less susceptible to such issues if it is capable of performing the task and was found to generalize well across various hyperparameter settings.

\subsection{Role of Self-Attention}

\begin{figure}[t]
	\begin{subfigure}{.45\textwidth}
		\centering
		\includegraphics[scale=0.12]{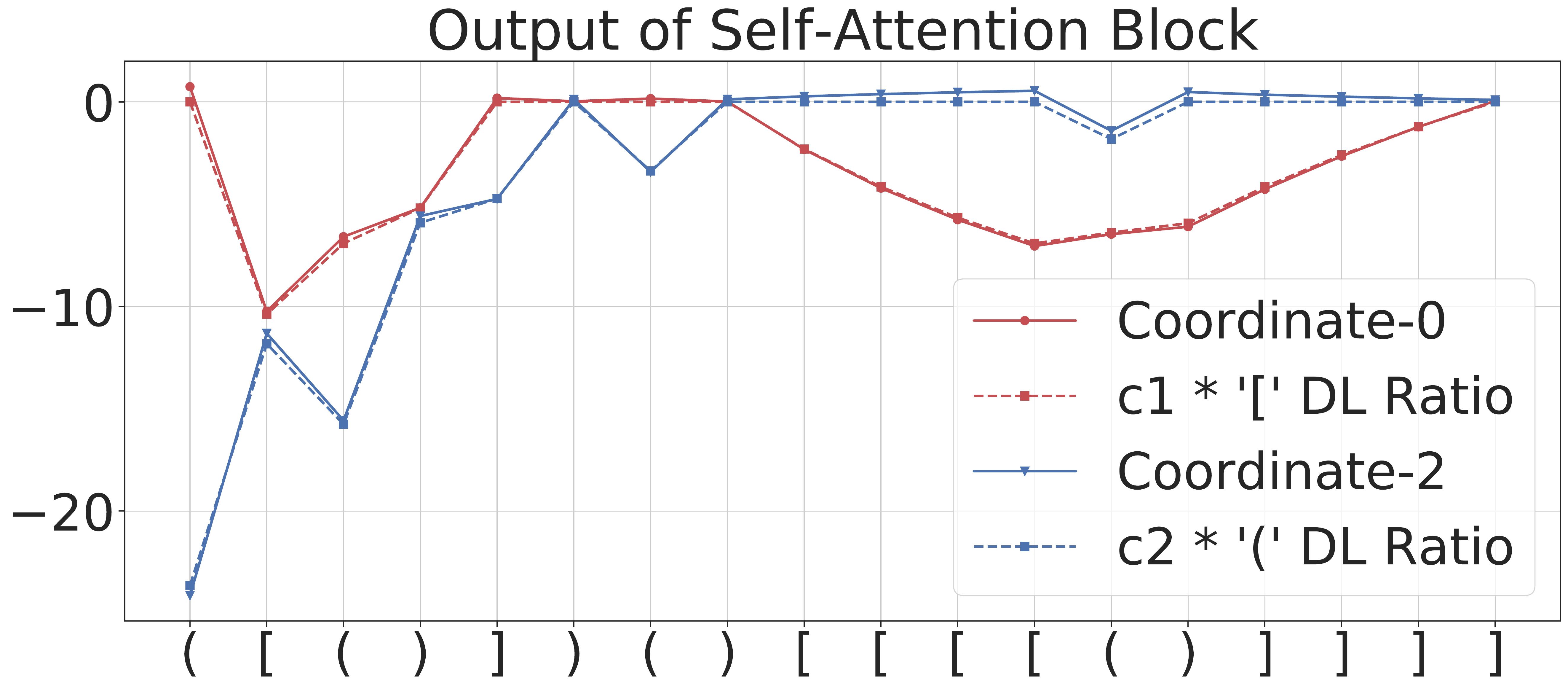}
		\caption{\label{fig:attn_out_shuff2} }
	\end{subfigure}
	\newline
	
	\begin{subfigure}{.45\textwidth}
		\centering
		\includegraphics[scale=0.12]{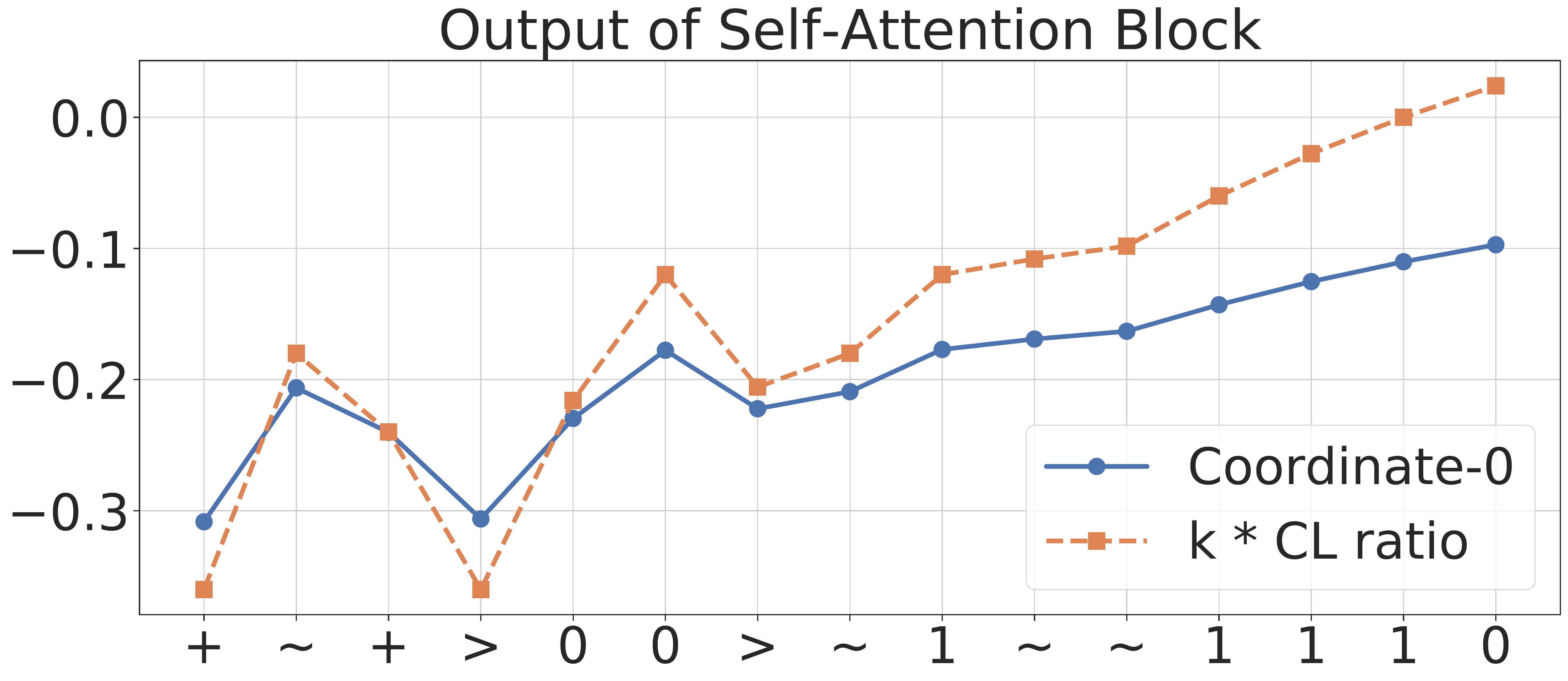}
		\caption{\label{fig:attn_out_bool} }
	\end{subfigure}
	\caption{\label{fig:attn_out} Values of different coordinates of the output of self-attention block of the models trained on Shuffle-2 and \bool-$3$. The dotted lines are the scaled depth to length ratios for Shuffle-2 and scaled counter value to length ratios for \bool-$3$. We observe a near perfect  Pearson correlation coefficent of 0.99 between outputs of self attention block and the DL and CL ratios.}
	
\end{figure}

In order to check our hypothesis in Sec. \ref{sec:const}, we visualize certain attributes of trained models that generalize well on Shuffle-2 and \bool-$3$.\footnote{We take the model with the smallest number of parameters that generalized well making it feasible for us to visualize it.}
Our construction in Sec. \ref{sec:const} recognizes sequences in \shuff{} by computing the depth to length ratio of the input at each step via self-attention mechanism. For \booln{}, the model can achieve the task similarly by computing the corresponding counter value divided by length (refer to Lemma \ref{lem:bool}). Interestingly, upon visualizing the outputs of the self-attention block for a model trained on Shuffle-2, we found a strong correlation of its elements with the depth to length ratio. As shown in Fig. \ref{fig:attn_out_shuff2}, different coordinates of the output vector of the self-attention block contain computations corresponding to different counters of the Shuffle-2 language. We observe the same behavior for models trained on Shuffle-4 language (refer to Figure \ref{fig:attn_out_shuff4} in appendix). Similarly, upon visualizing a model trained on Boolean Expressions with 3 operators, we found strong correlation\footnote{The Pearson correlation of values were $\sim0.99$ } between its elements and the ratio of the counter value and length of the input (refer to Figure \ref{fig:attn_out_bool}). This indicates that the model learns to recognize inputs by carrying out the required computation in an indirect manner, as described in our construction. Additionally, for both models, we found that the attention weights of the self-attention block were uniformly distributed (refer to Figure \ref{fig:attnwts} in appendix). Further, on inspecting the embedding and value vectors of the open and closing brackets, we found that their respective coordinates were opposite in sign and similar in magnitude. As opposed to \shuff{}, for \booln{}, the magnitudes of the elements in the value vectors were according to their corresponding arity. For instance, the magnitude for a ternary operator was (almost) thrice the magnitude for a unary operator (refer to Figure \ref{fig:values} in appendix). These observations are consistent with our construction, indicating that the model uses its value vectors to determine the counter updates and then at each step, aggregates all the values to obtain a form of the final counter value in an indirect manner. This is complementary to LSTMs, which can simulate the behavior of $k$-counters more directly by making respective updates to its cell states upon receiving each input \cite{suzgun2019lstm}.

\begin{table}[t]
	\scriptsize{\centering
		\begin{tabular}{P{9em}P{2.5em}P{2.5em}P{2.5em}P{2.5em}}
			\toprule
			&\multicolumn{2}{c}{\textbf{1-Layer}} &\multicolumn{2}{c}{\textbf{2-Layer}} \\ 
			\cmidrule(lr){2-3}\cmidrule(lr){4-5}
			\textbf{Model Type} & \textbf{Bin 0}& \textbf{Bin 1}& \textbf{Bin 0}& \textbf{Bin 1}\\
			\midrule
			Positional Masking & 45.1 & 38.9  &  100.0 & 99.2 \\
			Positional Encoding & 55.8 & 37.9  &  100.0 & 99.6  \\
			%	Relative Encoding & 63.6 & 84.6  & 93.8 & 65.4 \\
			\midrule
			LSTM & 100.0 & 100.0 & 100.0 &  100.0\\	
			\bottomrule
		\end{tabular}
		\caption{\label{tab:reset}Results on language {\ResetDyck} with different number of layers.}
	}
\end{table}

\subsection{Limitations of the Single-Layer Transformer}
Although we observed that single-layer Transformers are easily able to recognize some of the popularly studied counter languages, at the same time, it is not necessarily true for counter languages that require reset operations. We define a variant of the \dyck{} language. Let {\ResetDyck} be the language defined over the alphabet $\Sigma=\{[,],\#\}$, where $\#$ denotes a symbol that resets the counter. Words in {\ResetDyck} have the form $\Sigma^*\#v$, where the string $v$ belongs to Dyck-1. When the machine encounters the reset symbol $\#$, it must ignore all the previous input, reset the counter to $0$ and go to start state. It is easy to show that this cannot be directly implemented with a single layer self-attention network with positional masking (Lemma \ref{lem:reset} in Appendix). The key limitation for both with and without encodings is the fact that for a single layer network the scoring function $\langle Q(\vx_n), K(\vx_{\#}) \rangle$ and the value vector corresponding to the reset symbol is independent of the preceding inputs which it is supposed to negate (reset). The same limitation does not hold for multi-layer networks where the value vector, as well as the scoring function for the reset symbol, are dependent on its preceding inputs. On evaluating the model on data generated from such a language, we found that single-layer networks are unable to perform well in contrast to networks with two layers 
(Table \ref{tab:reset})\footnote{The results and limitations of single-layer Transformers are confined to this subsection. The rest of the results in the paper are not specific to single-layer Transformers unless explicitly mentioned.}. LSTMs, on the other hand, can emulate the reset operation using forget gate.

\section{Results on Regular Languages}\label{sec:res_reg}

\begin{table}[t]
	\scriptsize{\centering
		\begin{tabular}{P{6em}P{3em}P{3em}P{3em}P{3em}P{3em}}
			\toprule
			&&\multicolumn{2}{c}{\textbf{Transformer}} &\multicolumn{2}{c}{\textbf{LSTM}} \\ 
			\cmidrule(lr){3-4}\cmidrule(lr){5-6}
			\textbf{Language} & \textbf{Star-Free} & \textbf{Bin 0}& \textbf{Bin 1}& \textbf{Bin 0}& \textbf{Bin 1}\\
			\midrule
			Tomita 1 & \cmark& 100.0 & 100.0  & 100.0 & 100.0 \\
			Tomita 2 & \cmark & 100.0 & 100.0  & 100.0 & 100.0 \\
			Tomita 3 & \xmark& 75.4 & 10.8  &  100.0 & 100.0 \\
			Tomita 4 & \cmark& 100.0 & 92.4  &  100.0 & 100.0 \\
			Tomita 5 & \xmark &29.3 & 0.0  &  100.0 & 100.0 \\
			Tomita 6 & \xmark & 88.8 & 0.0  &  100.0 & 100.0 \\
			Tomita 7 & \cmark & 100.0 & 100.0  &  100.0 & 100.0 \\									
			\bottomrule
		\end{tabular}
		\caption{\label{tab:tomita}Results on Tomita grammar }
	}
\end{table}

We first examine the popular benchmark of Tomita grammars. While the LSTMs generalize perfectly on all 7 languages, Transformers are unable to generalize on 3 languages, all of which are non-star-free. Note that, all star-free languages in Tomita grammar have dot-depth 1. Recognizing non-star-free languages requires modeling properties such as periodicity and modular counting. Consequently, we evaluate the model on some of the simplest non-star-free languages such as the languages $(aa)^{*}$ and Parity. We find that they consistently fail to learn or generalize on such languages, whereas LSTMs of very small sizes perform flawlessly. Table \ref{tab:nsfree} lists the performance on some non-star-free languages.  Note that LSTMs can easily recognize such simple non-star-free languages considered here by using its internal memory and recurrence \footnote{For tasks such as Parity, LSTMs can simply flip between two values in its hidden state upon receiving $1$'s as input and ignore when it receives $0$'s as input.}. However, doing the same task via self-attention mechanism without using any internal memory could be highly non-trivial and potentially impossible. Languages such as $(aa)^*$ and Parity are among the simplest non-star-free languages, and hence limitations in recognizing such languages carry over to a larger class of languages. 
The results above may suggest that the star-free languages are precisely the regular languages recognizable by Transformers. As we will see in the next section, this is not so.

\subsection{Necessity of Positional Encodings}

\begin{table}[t]
	\scriptsize{\centering
		\begin{tabular}{P{3em}P{3.5em}P{4.5em}P{4.5em}P{2.5em}P{2.5em}}
			\toprule
			&&\multicolumn{2}{c}{\textbf{Transformer}} &\multicolumn{2}{c}{\textbf{LSTM}} \\ 
			\cmidrule(lr){3-4}\cmidrule(lr){5-6}
			\textbf{Language} & \textbf{Property} & \textbf{Bin 0}& \textbf{Bin 1}& \textbf{Bin 0}& \textbf{Bin 1}\\
			\midrule
			Parity & non-SF&  68.7 (23.0) & 0.0 (0.0)  &  100.0 & 100.0 \\
			$(aa)^{*}$ & non-SF &  100 (1.3) & 0.0 (0.0)  &  100.0 & 100.0  \\
			$(abab)^{*}$ & non-SF& 100.0 (9.9) & 5.4 (0.0)  & 100.0 & 100.0 \\
			$\dn{1}$ & depth-1 &100.0 & 100.0  & 100.0 & 100.0 \\
			$\dn{2}$&depth-2&74.6&3.1&100.0&100.0\\
			$\dn{4}$ & depth -4 &90.2 & 3.3 & 100.0 & 100.0\\
			%$(aaaa)^{*}$ &6.7& 0.0 & 100.0 &  100.0\\	
			\bottomrule
		\end{tabular}
		\caption{\label{tab:nsfree}Results on non-star-free languages (non-SF) and the language $\dn{n}$. The values in parenthesis correspond to the scores obtained for a model without residual connections. This is to prevent the model from solving the task by memorizing the positional encodings and study the ability of self-attention mechanism to solve the task. }
	}
\end{table}

The architecture of Transformer imposes limitations for recognizing certain types of languages. Although Transformers seem to generalize well when they are capable of performing a task with only positional masking, they are incapable of recognizing certain types of languages without explicit positional encodings. We consider the
family of star-free languages $\dn{n}$ defined in Sec. \ref{sec:formaldef}. Note that the task of recognizing $\dn{n}$ is equivalent to recognizing \dyck{} with maximum depth $n$, where the symbols $a$ and $b$ in $\dn{n}$ are analogous to open and closing brackets in \dyck{} respectively. The primary difference between recognizing $\dn{n}$ and \dyck{} is that in case of $\dn{n}$, when the input reaches the maximum depth $n$, the model must predict $a$ (the open bracket) as invalid for the next character, whereas in \dyck{}, open brackets are always allowed. We show that although Transformers with only positional masking can generalize well on \dyck{}, they are incapable of recognizing the language $\dn{n}$ for $n>1$. The limitation arises from the fact that when the model receives a sequence of only $a$'s, then due to the softmax based aggregation, the output of the self-attention block $\va_i$ will be a constant vector, implying that the output of the feed-forward will also be a constant vector, that is, $\vz_1 = \vz_2=\ldots=\vz_n$. In case of languages such as $\dn{n}$, if the input begins with $n$ consecutive $a$s, then, since the model cannot distinguish between the $n$-th $a$ and the preceding $a$'s, the model cannot recognize the language $\dn{n}$. This limitation does not exist if the model is provided explicit positional encoding. Upon evaluating Transformers with positional encodings on instances of the language $\dn{n}$, we found that the models are able to generalize to a certain extent on strings within the same lengths as seen during training but fail to generalize on higher lengths (Table \ref{tab:nsfree}). It is perhaps surprising that small and simpler self-attention networks can generalize very well on languages such as \dyck{} but achieve limited performance on a language that belongs to a much simpler class such as star-free.

Similarly, since $(aa)^{*}$, is a unary language (alphabet size is $1$), the model will always receive the same character at each step. Hence, for a model with only positional masking, the output vector will be the same at every step, making it incapable of recognizing the language $(aa)^*$.  For the language Parity, when the input word contains only $1$'s, the task reduces to recognizing $(11)^*$ and hence a model without positional encodings is incapable of recognizing Parity even for very small lengths regardless of the size of the network (refer to Lemma \ref{lem:reg_pos}). We find it surprising that for 
Parity, which is permutation invariant, positional encodings are necessary for transformers to recognize them even for very small lengths.

\begin{table}[t]
	\scriptsize {\centering
		\begin{tabular}{P{9em}P{2.5em}P{2.5em}P{2.5em}P{2.5em}}
			\toprule
			&\multicolumn{2}{c}{\textbf{$(aa)^*$}} &\multicolumn{2}{c}{\textbf{$(aaaa)^*$}} \\ 
			\cmidrule(lr){2-3}\cmidrule(lr){4-5}
			\textbf{Encoding Scheme} & \textbf{Bin 0}& \textbf{Bin 1}& \textbf{Bin 0}& \textbf{Bin 1}\\
			\midrule
			Positional Masking & 0.0 & 0.0  &  0.0 & 0.0 \\
			Absolute Encoding & 1.3 & 0.0  &  6.7 & 0.0  \\
			Relative Encoding & 0.6 & 0.0  & 1.7 & 0.0 \\
			$\textsf{cos}(n\pi)$ & 100.0 & 100.0 & 0.0 &  0.0\\
			Trainable Embedding & 100.0 & 0.0 & 100.0  & 0.0 \\		
			\bottomrule
		\end{tabular}
		
		\caption{\label{tab:posit} Performance of transformer based models on $(aa)^*$ and $(aaaa)^*$, for different types of position encoding schemes. To separately study the effect of different position encodings on the self attention mechanism, we do not include residual connections in the models studied here.}
	}
\end{table}

\subsection{Influence of Custom Positional Encodings}
The capability and complexity of the network could significantly depend on the positional encoding scheme. For instance, for language $(aa)^*$, the ability of a self-attention network to recognize it depends solely on the positional encoding. Upon evaluating with standard absolute and relative encoding schemes, we observe that the model is unable to learn or generalize well. At the same time, it is easy to show that if $\textsf{cos}(n\pi)$, which has a period of two is used as positional encoding, the self-attention mechanism can easily achieve the task which we also observe when we empirically evaluated with such an encoding. However, the same encoding would not work for a language such as $(aaaa)^*$, which has a periodicity of four.  Table \ref{tab:posit} shows the performance of the model with different types of encodings. When we used fixed-length trainable positional embeddings, the obtained learned embeddings were very similar to the $\textsf{cos}(n\pi)$ form; however, such embeddings cannot be used for sequences of higher lengths. This also raises the need for better learnable encodings schemes that can extrapolate to variable lengths of inputs not seen during training data such as \cite{liu2020learning}.

Our experiments on over 15 regular languages seem to indicate that Transformers are able to generalize on star-free languages within dot-depth $1$ but have difficulty with higher dot-depths or more complex classes like non-star-free languages. Table \ref{tab:regresults} in Appendix lists results on all considered regular languages.

\section{Discussion}\label{sec:discussion}

We showed that Transformers can easily generalize on certain counter languages such as Shuffle-Dyck and Boolean Expressions in a manner similar to our proposed construction. Our visualizations imply that Transformers do so with a generalizable mechanism instead of overfitting on some statistical regularities. Similar to natural languages, Boolean Expressions consist of recursively nested hierarchical constituents. Recently, \citet{papadimitriou2020pretraining} showed that pretraining LSTMs on formal languages like Shuffle-Dyck transfers to LM performance on natural languages. At the same time, our results show clear limitations of Transformers compared to LSTMs on a large class of regular languages. Evidently, the performance and capabilities of Transformers heavily depend on architectural constituents e.g., the positional encoding schemes and the number of layers. Recurrent models have a more automata-like structure well-suited for counter and regular languages, whereas self-attention networks' structure is very different, which seems to limit their abilities for the considered tasks. 

Our work poses a number of open questions. Our results are consistent with the hypothesis that Transformers  generalize well for star-free languages with dot-depth 1, but not for higher depths. Clarifying this hypothesis theoretically and empirically is an attractive challenge. What does the disparity between the performance of Transformers on natural and formal languages indicate about the complexity of natural languages and their relation to linguistic analysis? (See also \citet{Hahn}). Another interesting direction would be to understand whether certain modifications or recently proposed variants of Transformers improve their performance on formal languages. Regular and counter languages model some aspects of natural language while context-free languages model other aspects such as hierarchical dependencies. Although our results have some implications on them, we leave a detailed study on context-free languages for future work.

\section*{Acknowledgements}
We thank the anonymous reviewers for their constructive comments and suggestions. We would also like to thank our colleagues at Microsoft Research and Michael Hahn for their valuable feedback and helpful discussions.

\bibliography{emnlp2020}

\begin{thebibliography}{41}
\expandafter\ifx\csname natexlab\endcsname\relax\def\natexlab#1{#1}\fi

\bibitem[{Bhattamishra et~al.(2020)Bhattamishra, Patel, and
  Goyal}]{bhattamishra2020computational}
Satwik Bhattamishra, Arkil Patel, and Navin Goyal. 2020.
\newblock On the computational power of transformers and its implications in
  sequence modeling.
\newblock \emph{arXiv preprint arXiv:2006.09286}.

\bibitem[{Cohen and Brzozowski(1971)}]{Cohen-Brzozowski}
Rina Cohen and Janusz Brzozowski. 1971.
\newblock \href {https://doi.org/10.1016/S0022-0000(71)80003-X} {Dot-depth of
  star-free events.}
\newblock \emph{Journal of Computer and System Sciences}, 5:1--16.

\bibitem[{Dai et~al.(2019)Dai, Yang, Yang, Carbonell, Le, and
  Salakhutdinov}]{dai-etal-2019-transformer}
Zihang Dai, Zhilin Yang, Yiming Yang, Jaime Carbonell, Quoc Le, and Ruslan
  Salakhutdinov. 2019.
\newblock \href {https://doi.org/10.18653/v1/P19-1285} {Transformer-{XL}:
  Attentive language models beyond a fixed-length context}.
\newblock In \emph{Proceedings of the 57th Annual Meeting of the Association
  for Computational Linguistics}, pages 2978--2988, Florence, Italy.
  Association for Computational Linguistics.

\bibitem[{Devlin et~al.(2019)Devlin, Chang, Lee, and
  Toutanova}]{devlin-etal-2019-bert}
Jacob Devlin, Ming-Wei Chang, Kenton Lee, and Kristina Toutanova. 2019.
\newblock \href {https://doi.org/10.18653/v1/N19-1423} {{BERT}: Pre-training of
  deep bidirectional transformers for language understanding}.
\newblock In \emph{Proceedings of the 2019 Conference of the North {A}merican
  Chapter of the Association for Computational Linguistics: Human Language
  Technologies, Volume 1 (Long and Short Papers)}, pages 4171--4186,
  Minneapolis, Minnesota. Association for Computational Linguistics.

\bibitem[{Diekert and Gastin(2008)}]{Diekert-Gastin}
Volker Diekert and Paul Gastin. 2008.
\newblock First-order definable languages.
\newblock In \emph{Logic and Automata: History and Perspectives [in Honor of
  Wolfgang Thomas]}, volume~2 of \emph{Texts in Logic and Games}, pages
  261--306. Amsterdam University Press.

\bibitem[{Fischer et~al.(1968)Fischer, Meyer, and
  Rosenberg}]{fischer1968counter}
Patrick~C Fischer, Albert~R Meyer, and Arnold~L Rosenberg. 1968.
\newblock Counter machines and counter languages.
\newblock \emph{Mathematical systems theory}, 2(3):265--283.

\bibitem[{Gers and Schmidhuber(2001)}]{gers2001lstm}
Felix~A Gers and E~Schmidhuber. 2001.
\newblock Lstm recurrent networks learn simple context-free and
  context-sensitive languages.
\newblock \emph{IEEE Transactions on Neural Networks}, 12(6):1333--1340.

\bibitem[{Hahn(2020)}]{Hahn}
Michael Hahn. 2020.
\newblock \href {https://doi.org/10.1162/tacl\_a\_00306} {Theoretical
  limitations of self-attention in neural sequence models}.
\newblock \emph{Transactions of the Association for Computational Linguistics},
  8:156--171.

\bibitem[{Hochreiter and Schmidhuber(1997)}]{hochreiter1997long}
Sepp Hochreiter and J{\"u}rgen Schmidhuber. 1997.
\newblock Long short-term memory.
\newblock \emph{Neural computation}, 9(8):1735--1780.

\bibitem[{J{\"a}ger and Rogers(2012)}]{Jaeger-Rogers}
Gerhard J{\"a}ger and James Rogers. 2012.
\newblock Formal language theory: refining the chomsky hierarchy.
\newblock \emph{Philosophical Transactions of the Royal Society B: Biological
  Sciences}, 367(1598):1956--1970.

\bibitem[{Kolen and Kremer(2001)}]{kolen2001field}
John~F Kolen and Stefan~C Kremer. 2001.
\newblock \emph{A field guide to dynamical recurrent networks}.
\newblock John Wiley \& Sons.

\bibitem[{Korsky and Berwick(2019)}]{korsky2019computational}
Samuel~A Korsky and Robert~C Berwick. 2019.
\newblock On the computational power of rnns.
\newblock \emph{arXiv preprint arXiv:1906.06349}.

\bibitem[{Liu et~al.(2020)Liu, Yu, Dhillon, and Hsieh}]{liu2020learning}
Xuanqing Liu, Hsiang-Fu Yu, Inderjit Dhillon, and Cho-Jui Hsieh. 2020.
\newblock Learning to encode position for transformer with continuous dynamical
  model.
\newblock \emph{arXiv preprint arXiv:2003.09229}.

\bibitem[{Liu et~al.(2019)Liu, Ott, Goyal, Du, Joshi, Chen, Levy, Lewis,
  Zettlemoyer, and Stoyanov}]{liu2019roberta}
Yinhan Liu, Myle Ott, Naman Goyal, Jingfei Du, Mandar Joshi, Danqi Chen, Omer
  Levy, Mike Lewis, Luke Zettlemoyer, and Veselin Stoyanov. 2019.
\newblock Roberta: A robustly optimized bert pretraining approach.
\newblock \emph{arXiv preprint arXiv:1907.11692}.

\bibitem[{McNaughton and Papert(1971)}]{McNaughton-Papert}
Robert McNaughton and Seymour~A. Papert. 1971.
\newblock \emph{Counter-Free Automata (M.I.T. Research Monograph No. 65)}.
\newblock The MIT Press.

\bibitem[{Merrill(2019)}]{merrill2019sequential}
William Merrill. 2019.
\newblock \href {https://doi.org/10.18653/v1/W19-3901} {Sequential neural
  networks as automata}.
\newblock In \emph{Proceedings of the Workshop on Deep Learning and Formal
  Languages: Building Bridges}, pages 1--13, Florence. Association for
  Computational Linguistics.

\bibitem[{Merrill(2020)}]{merrill2020linguistic}
William Merrill. 2020.
\newblock On the linguistic capacity of real-time counter automata.
\newblock \emph{arXiv preprint arXiv:2004.06866}.

\bibitem[{Merrill et~al.(2020)Merrill, Weiss, Goldberg, Schwartz, Smith, and
  Yahav}]{merrill2020formal}
William Merrill, Gail Weiss, Yoav Goldberg, Roy Schwartz, Noah~A Smith, and
  Eran Yahav. 2020.
\newblock A formal hierarchy of rnn architectures.
\newblock \emph{arXiv preprint arXiv:2004.08500}.

\bibitem[{Michalenko et~al.(2019)Michalenko, Shah, Verma, Chaudhuri, and
  Patel}]{michalenko2019representing}
Joshua~J. Michalenko, Ameesh Shah, Abhinav Verma, Swarat Chaudhuri, and
  Ankit~B. Patel. 2019.
\newblock \href {https://openreview.net/forum?id=H1zeHnA9KX} {Finite automata
  can be linearly decoded from language-recognizing {RNN}s}.
\newblock In \emph{International Conference on Learning Representations}.

\bibitem[{Papadimitriou and Jurafsky(2020)}]{papadimitriou2020pretraining}
Isabel Papadimitriou and Dan Jurafsky. 2020.
\newblock Pretraining on non-linguistic structure as a tool for analyzing
  learning bias in language models.
\newblock \emph{arXiv preprint arXiv:2004.14601}.

\bibitem[{P{\'e}rez et~al.(2019)P{\'e}rez, Marinkovi{\'c}, and
  Barcel{\'o}}]{perez2019turing}
Jorge P{\'e}rez, Javier Marinkovi{\'c}, and Pablo Barcel{\'o}. 2019.
\newblock \href {https://openreview.net/forum?id=HyGBdo0qFm} {On the turing
  completeness of modern neural network architectures}.
\newblock In \emph{International Conference on Learning Representations}.

\bibitem[{Pin(2017)}]{Pin}
Jean-Éric Pin. 2017.
\newblock \href {https://doi.org/10.1142/9789813148208_0008} {The dot-depth
  hierarchy, 45 years later}.
\newblock In \emph{The Role of Theory in Computer Science}, pages 177--201.

\bibitem[{Radford et~al.(2018)Radford, Narasimhan, Salimans, and
  Sutskever}]{radford2018improving}
Alec Radford, Karthik Narasimhan, Tim Salimans, and Ilya Sutskever. 2018.
\newblock Improving language understanding by generative pre-training.
\newblock \emph{URL https://s3-us-west-2. amazonaws.
  com/openai-assets/researchcovers/languageunsupervised/language understanding
  paper. pdf}.

\bibitem[{Reif et~al.(2019)Reif, Yuan, Wattenberg, Viegas, Coenen, Pearce, and
  Kim}]{reif2019visualizing}
Emily Reif, Ann Yuan, Martin Wattenberg, Fernanda~B Viegas, Andy Coenen, Adam
  Pearce, and Been Kim. 2019.
\newblock Visualizing and measuring the geometry of bert.
\newblock In \emph{Advances in Neural Information Processing Systems}, pages
  8592--8600.

\bibitem[{Rogers et~al.(2020)Rogers, Kovaleva, and
  Rumshisky}]{rogers2020primer}
Anna Rogers, Olga Kovaleva, and Anna Rumshisky. 2020.
\newblock A primer in bertology: What we know about how bert works.
\newblock \emph{arXiv preprint arXiv:2002.12327}.

\bibitem[{Sennhauser and Berwick(2018)}]{sennhauser-berwick-2018-evaluating}
Luzi Sennhauser and Robert Berwick. 2018.
\newblock \href {https://doi.org/10.18653/v1/W18-5414} {Evaluating the ability
  of {LSTM}s to learn context-free grammars}.
\newblock In \emph{Proceedings of the 2018 {EMNLP} Workshop {B}lackbox{NLP}:
  Analyzing and Interpreting Neural Networks for {NLP}}, pages 115--124,
  Brussels, Belgium. Association for Computational Linguistics.

\bibitem[{Shen et~al.(2018)Shen, Zhou, Long, Jiang, Pan, and
  Zhang}]{Shen2018DiSANDS}
Tao Shen, Tianyi Zhou, Guodong Long, Jing Jiang, Shirui Pan, and C.~Zhang.
  2018.
\newblock Disan: Directional self-attention network for rnn/cnn-free language
  understanding.
\newblock In \emph{AAAI}.

\bibitem[{Skachkova et~al.(2018)Skachkova, Trost, and
  Klakow}]{skachkova-etal-2018-closing}
Natalia Skachkova, Thomas Trost, and Dietrich Klakow. 2018.
\newblock \href {https://doi.org/10.18653/v1/W18-5425} {Closing brackets with
  recurrent neural networks}.
\newblock In \emph{Proceedings of the 2018 {EMNLP} Workshop {B}lackbox{NLP}:
  Analyzing and Interpreting Neural Networks for {NLP}}, pages 232--239,
  Brussels, Belgium. Association for Computational Linguistics.

\bibitem[{Straubing(1994)}]{Straubing}
Howard Straubing. 1994.
\newblock \emph{Finite Automata, Formal Logic, and Circuit Complexity}.
\newblock Birkhauser Verlag, CHE.

\bibitem[{Suzgun et~al.(2019{\natexlab{a}})Suzgun, Belinkov, Shieber, and
  Gehrmann}]{suzgun2019lstm}
Mirac Suzgun, Yonatan Belinkov, Stuart Shieber, and Sebastian Gehrmann.
  2019{\natexlab{a}}.
\newblock \href {https://doi.org/10.18653/v1/W19-3905} {{LSTM} networks can
  perform dynamic counting}.
\newblock In \emph{Proceedings of the Workshop on Deep Learning and Formal
  Languages: Building Bridges}, pages 44--54, Florence. Association for
  Computational Linguistics.

\bibitem[{Suzgun et~al.(2019{\natexlab{b}})Suzgun, Belinkov, and
  Shieber}]{suzgun2018evaluating}
Mirac Suzgun, Yonatan Belinkov, and Stuart~M. Shieber. 2019{\natexlab{b}}.
\newblock \href {https://doi.org/10.7275/s02b-4d91} {On evaluating the
  generalization of {LSTM} models in formal languages}.
\newblock In \emph{Proceedings of the Society for Computation in Linguistics
  ({SC}i{L}) 2019}, pages 277--286.

\bibitem[{Tsai et~al.(2019)Tsai, Bai, Yamada, Morency, and
  Salakhutdinov}]{tsai2019transformer}
Yao-Hung~Hubert Tsai, Shaojie Bai, Makoto Yamada, Louis-Philippe Morency, and
  Ruslan Salakhutdinov. 2019.
\newblock \href {https://doi.org/10.18653/v1/D19-1443} {Transformer dissection:
  An unified understanding for transformer{'}s attention via the lens of
  kernel}.
\newblock In \emph{Proceedings of the 2019 Conference on Empirical Methods in
  Natural Language Processing and the 9th International Joint Conference on
  Natural Language Processing (EMNLP-IJCNLP)}, pages 4344--4353, Hong Kong,
  China. Association for Computational Linguistics.

\bibitem[{Vaswani et~al.(2017)Vaswani, Shazeer, Parmar, Uszkoreit, Jones,
  Gomez, Kaiser, and Polosukhin}]{vaswani2017attention}
Ashish Vaswani, Noam Shazeer, Niki Parmar, Jakob Uszkoreit, Llion Jones,
  Aidan~N Gomez, {\L}ukasz Kaiser, and Illia Polosukhin. 2017.
\newblock Attention is all you need.
\newblock In \emph{Advances in neural information processing systems}, pages
  5998--6008.

\bibitem[{Voita et~al.(2019)Voita, Talbot, Moiseev, Sennrich, and
  Titov}]{voita-etal-2019-analyzing}
Elena Voita, David Talbot, Fedor Moiseev, Rico Sennrich, and Ivan Titov. 2019.
\newblock \href {https://doi.org/10.18653/v1/P19-1580} {Analyzing multi-head
  self-attention: Specialized heads do the heavy lifting, the rest can be
  pruned}.
\newblock In \emph{Proceedings of the 57th Annual Meeting of the Association
  for Computational Linguistics}, pages 5797--5808, Florence, Italy.
  Association for Computational Linguistics.

\bibitem[{Wang et~al.(2018{\natexlab{a}})Wang, Zhang, II, Xing, Liu, and
  Giles}]{Comparative-Giles}
Qinglong Wang, Kaixuan Zhang, Alexander G.~Ororbia II, Xinyu Xing, Xue Liu, and
  C.~Lee Giles. 2018{\natexlab{a}}.
\newblock \href {http://arxiv.org/abs/1801.05420v2} {A comparative study of
  rule extraction for recurrent neural networks}.
\newblock \emph{CoRR}, abs/1801.05420v2.

\bibitem[{Wang et~al.(2018{\natexlab{b}})Wang, Zhang, Ororbia, Alexander, Xing,
  Liu, and Giles}]{wang2018comparative}
Qinglong Wang, Kaixuan Zhang, II~Ororbia, G~Alexander, Xinyu Xing, Xue Liu, and
  C~Lee Giles. 2018{\natexlab{b}}.
\newblock A comparative study of rule extraction for recurrent neural networks.
\newblock \emph{arXiv preprint arXiv:1801.05420}.

\bibitem[{Warstadt et~al.(2019)Warstadt, Cao, Grosu, Peng, Blix, Nie, Alsop,
  Bordia, Liu, Parrish, Wang, Phang, Mohananey, Htut, Jeretic, and
  Bowman}]{warstadt-etal-2019-investigating}
Alex Warstadt, Yu~Cao, Ioana Grosu, Wei Peng, Hagen Blix, Yining Nie, Anna
  Alsop, Shikha Bordia, Haokun Liu, Alicia Parrish, Sheng-Fu Wang, Jason Phang,
  Anhad Mohananey, Phu~Mon Htut, Paloma Jeretic, and Samuel~R. Bowman. 2019.
\newblock \href {https://doi.org/10.18653/v1/D19-1286} {Investigating
  {BERT}{'}s knowledge of language: Five analysis methods with {NPI}s}.
\newblock In \emph{Proceedings of the 2019 Conference on Empirical Methods in
  Natural Language Processing and the 9th International Joint Conference on
  Natural Language Processing (EMNLP-IJCNLP)}, pages 2877--2887, Hong Kong,
  China. Association for Computational Linguistics.

\bibitem[{Weiss et~al.(2018)Weiss, Goldberg, and Yahav}]{weiss2018practical}
Gail Weiss, Yoav Goldberg, and Eran Yahav. 2018.
\newblock \href {https://doi.org/10.18653/v1/P18-2117} {On the practical
  computational power of finite precision {RNN}s for language recognition}.
\newblock In \emph{Proceedings of the 56th Annual Meeting of the Association
  for Computational Linguistics (Volume 2: Short Papers)}, pages 740--745,
  Melbourne, Australia. Association for Computational Linguistics.

\bibitem[{Weiss et~al.(2019)Weiss, Goldberg, and Yahav}]{weiss2019learning}
Gail Weiss, Yoav Goldberg, and Eran Yahav. 2019.
\newblock \href
  {http://papers.nips.cc/paper/9062-learning-deterministic-weighted-automata-with-queries-and-counterexamples.pdf}
  {Learning deterministic weighted automata with queries and counterexamples}.
\newblock In H.~Wallach, H.~Larochelle, A.~Beygelzimer, F.~AlcheBuc, E.~Fox,
  and R.~Garnett, editors, \emph{Advances in Neural Information Processing
  Systems 32}, pages 8560--8571. Curran Associates, Inc.

\bibitem[{Yang et~al.(2019)Yang, Wang, Wong, Chao, and Tu}]{yang2019assessing}
Baosong Yang, Longyue Wang, Derek~F. Wong, Lidia~S. Chao, and Zhaopeng Tu.
  2019.
\newblock \href {https://doi.org/10.18653/v1/P19-1354} {Assessing the ability
  of self-attention networks to learn word order}.
\newblock In \emph{Proceedings of the 57th Annual Meeting of the Association
  for Computational Linguistics}, pages 3635--3644, Florence, Italy.
  Association for Computational Linguistics.

\bibitem[{Yun et~al.(2020)Yun, Bhojanapalli, Rawat, Reddi, and
  Kumar}]{yun2019transformers}
Chulhee Yun, Srinadh Bhojanapalli, Ankit~Singh Rawat, Sashank Reddi, and Sanjiv
  Kumar. 2020.
\newblock \href {https://openreview.net/forum?id=ByxRM0Ntvr} {Are transformers
  universal approximators of sequence-to-sequence functions?}
\newblock In \emph{International Conference on Learning Representations}.

\end{thebibliography}
\bibliographystyle{acl_natbib}

\newpage
\clearpage
\appendix

\section{Roadmap}
 The appendix is organized as follows. In section \ref{sec:definitions} we first provide formal definitions of the key languages used in our investigation in the main paper. In sections \ref{subsec:counter_automata} and \ref{subsec:dot-depth}, we also provide the formal definitions of automata, star-free languages and the dot-depth hierarchy. In section \ref{sec:expr_res}, we provide the details of all our expressiveness results. Section \ref{sec:exp_details} contains the details of our experimental setup which could be relevant for reproducibility of the results and includes a thorough discussion of the choice of character prediction task. The list of all the formal languages we have considered, their dataset statistics as well as the results are provided in section \ref{sec:exp_details}.

\section{Definitions}\label{sec:definitions}
In this section, we provide formal definitions of some of the languages used in our analysis. In counter languages, we first define the family of shuffled Dyck-1 languages. The language Dyck-1 is a simple context-free language that can also be recognized by a counter automaton with a single counter. We generate the data for Dyck-1 based on the following PCFG,

\begin{align*}
S \rightarrow \begin{cases} 
( S ) & \text{with probability } p \\
SS & \text{with probability } q \\ 
\varepsilon & \text{with probability } 1 - (p+q) 
\end{cases}
\end{align*}
\noindent where $0 < p, q < 1$ and $(p+q) < 1$. We use $0.5$ as the value of $p$ and $0.25$ as the value for $q$. 

\noindent \textbf{Shuffle-Dyck.} We now define the Shuffle-Dyck language introduced and described in \cite{suzgun2019lstm}. We first define the shuffling operation formally. The shuffling operation $||: \Sigma^{*} \times \Sigma^{*} \to \mathcal{P} (\Sigma^{*})$ can be inductively defined as follows:%
\footnote{We abuse notation by allowing a string to stand for the singleton containing that string. $\epsilon$ is the empty string.}
\vspace{-0.2em}
\begin{itemize}
	\setlength\itemsep{-0.2em}
	\item $u || \varepsilon = \varepsilon || u = \{u\}$
	\item $\alpha u || \beta v = \alpha (u || \beta v) \cup \beta (\alpha u || v)$
\end{itemize}
\noindent for any $\alpha, \beta \in \Sigma$ and $u, v \in \Sigma^{*}$. For instance, the shuffle of $ab$ and $cd$ is
\vspace{-.2em}
\begin{align*}
ab || cd = \{abcd, acbd, acdb, cabd, cadb, cdab\}.
\end{align*}
There is a natural extension of the shuffling operation $||$ to languages.  The \textit{shuffle} of two languages $\mathcal{L}_1$ and $\mathcal{L}_2$, denoted $\mathcal{L}_1 || \mathcal{L}_2$, is the set of all possible interleavings of the elements of $\mathcal{L}_1$ and $\mathcal{L}_2$, respectively, that is:
\begin{align*}
\mathcal{L}_1 || \mathcal{L}_2 = \bigcup_{\substack{u \in \mathcal{L}_1, \ v \in \mathcal{L}_2}} u || v
\end{align*}
Given a language $\mathcal{L}$, we define its self-shuffling $\mathcal{L}||^2$ to be $\mathcal{L} || \sigma(\mathcal{L})$, where $\sigma$ is an isomorphism on the vocabulary of $\mathcal{L}$ to a disjoint vocabulary. More generally, we define the $k$-self-shuffle \[\mathcal{L}||^k = \left\{
\begin{array}{ll}
\{\varepsilon\}  & \mbox{if $k=0$}  \\
\mathcal{L} || \sigma(\mathcal{L}||^{k-1})  & \mbox{otherwise}\eqpunc{.}
\end{array}\right.\]

We use Shuffle-$k$ to denote the shuffle of $k$ Dyck-$1$ languages (Dyck-$1||^{k}$) each with its own brackets. Shuffle-1 is the same as Dyck-1. For instance the language Shuffle-$2$ is the shuffle of Dyck-$1$ over alphabet $\Sigma=\{(,)\}$ and another Dyck-1 over the alphabet $\Sigma=\{[,]\}$. Hence the resulting Shuffle-2 language is defined over alphabet $\Sigma= \{[, ], (,  )\}$ and contains words such as $([)]$ and $[((]))$ but not $])[($. This is different from the context-free language Dyck-2 in which $([])$ belongs to the language but $([)]$ does not. Similar to \cite{suzgun2019lstm} we generate the training data by generating sequence for Dyck-$n$ but by providing the correct target values for the character prediction task.

\noindent \textbf{$n$-ary Boolean Expressions.} We now define the family of languages $n$-ary Boolean Expressions parameterized by the number and arities of its operators. An instance of the language contains operators of different arities and as shown in \cite{fischer1968counter}, these languages can be recognized by counter-machines with a single counter. However as opposed to Dyck-1 the values with which the counters will be incremented or decremented will depend on the arity of its operator. A language with $n$ operators can be defined by the following derivation rules

\begin{verbatim}
<exp> -> <VALUE>
<exp> -> <UNARY> <exp>
<exp> -> <BINARY> <exp> <exp>
..
<exp> -> <n-ARY> <exp> .. <exp>
\end{verbatim}

\noindent \textbf{Tomita Grammars} Tomita Grammars are 7 regular langauges defined on the alphabet $\Sigma = \{0, 1\}$. Tomita-1 has the regular expression $1^*$ i.e. the strings containing only 1's and no 0s are allowed. Tomita-2 is defined by the regular expression $(10)^*$. Tomita-3 accepts the strings where odd number of consecutive 1s are always followed by an even number of 0s. Tomita-4 accepts the strings that do not contain 3 consecutive 0s. In Tomita-5 only the strings containing an even number of 0s and even number of 1s are allowed. In Tomita-6 the difference in the number of 1s and 0s should be divisible by 3 and finally, Tomita-7  has the regular expression $0^*1^*0^*1^*$.

We note that Tomita 2 $= \dn{1} = (01)^*$ and that the very simple language $\{0,1,2\}^*02^*$ has dot-depth 2 \cite{Cohen-Brzozowski}.

\subsection{Counter Automata}\label{subsec:counter_automata}

We define the general counter machine following \cite{fischer1968counter}. We are concerned with real-time counter machines here in which the number of computation steps is bounded by the number of inputs similar to how we use sequence models in practice. The machine has a finite number of unbounded counters and it modifies it by adding or subtracting values or resetting the counter value to $0$. For $m \in \mathbb{Z}$, let ${+}m$ denote the function 
$x \mapsto x+m$. Let $\times 0$ denote the constant zero function $x \mapsto 0$.

%The first counter automaton we introduce is the \textit{general counter machine}. This machine manipulates its counters by adding or subtracting from them. Later, we define other variants of this general automaton. For $m \in \mathbb{Z}$, let ${\pm}m$ denote the function $\lambda x. x \pm m$. Let $\times 0$ denote the constant zero function $\lambda x. 0$.

\begin{definition}[General counter machine \cite{fischer1968counter}]
	A $k$-counter machine is a tuple $\langle \Sigma, Q, q_0, u, \delta, F \rangle$ with
	\begin{enumerate}
		\item A finite alphabet $\Sigma$
		\item A finite set of states $Q$
		\item An initial state $q_0$
		\item A counter update function
		\begin{align*}
		u : \Sigma \times Q \times \{0, 1\}^k  \rightarrow \\
		\big( \{+m : m \in \mathbb{Z} \} \cup \{ \times 0 \} \big)^k
		\end{align*}
		\item A state transition function
		\begin{equation*}
		\delta : \Sigma \times Q \times \{0, 1\}^k \rightarrow Q
		\end{equation*}
		\item An acceptance mask
		\begin{equation*}
		F \subseteq Q \times \{0, 1\}^k
		\end{equation*}
	\end{enumerate}
\end{definition}

A machine processes an input string $x$ one token at a time.
For each token, we use $u$ to update the counters and $\delta$ to update the state according to the current input token, the current state, and a finite mask of the current counter values. 

For a vector $\vv $, let $z(\vv)$ denote the broadcasted ``zero-check" function, i.e. $z(\vv)_i$ is $0$ if $v_i=0$ or $1$ otherwise. Let $\langle q, \vc \rangle \in Q \times \mathbb{Z}^k$ be a configuration of machine $M$. Upon reading input $x_t \in \Sigma$, we define the transition
\begin{equation*}
\langle q, \vc \rangle \rightarrow_{x_t} \langle
\delta(x_t, q, z(\vc )) ,
u(x_t, q, z(\vc))(\vc)
\rangle.
\end{equation*}	

For any string $x \in \Sigma^*$ with length $n$, a counter machine accepts $x$ if there exist states $q_1, .., q_n$ and counter configurations $\vc_1, .., \vc_n$ such that
\begin{equation*}
\langle q_0, \vzero  \rangle \rightarrow_{x_1} \langle q_1, \vc_1 \rangle \rightarrow_{x_2} .. \rightarrow_{x_n} \langle q_n, \vc_n \rangle \in F .
\end{equation*}

A counter machines accepts a language $L$ if, for each $x \in \Sigma^*$, it accepts $x$ iff $x \in L$. Refer to \cite{merrill2020linguistic} for more details on counter machines, variants and their properties.

\subsection{Star-free regular languages and the dot-depth hierarchy}\label{subsec:dot-depth}
Star-free regular languages (defined in the main paper) are a simpler subclass of regular languages; they have regular expressions without Kleene star (but use set complementation). The set of star-free languages is further stratified by the dot-depth hierarchy, which is a hierarchy of families of languages whose union is the family of star-free languages. Informally, the position of a language in this hierarchy is a measure of the number of nested concatenations or sequentiality required to express the language in a star-free regular expression. 
Both the star-free regular languages as well as the dot-depth hierarchy are well-studied with rich connections and multiple (equivalent) definitions. 
For more information, see e.g. \cite{McNaughton-Papert,Cohen-Brzozowski,Straubing,Diekert-Gastin,Jaeger-Rogers,Pin}.

To define the dot-depth hierarchy, we first define Boolean and concatenation closures of language families.
For a language family $\mathcal{L}$ over a finite alphabet $\Sigma = \{a_1, \ldots, a_k\}$, its Boolean closure $\mathcal{B}\mathcal{L}$ is the set of languages obtained by applying Boolean operators (union, intersection and set complementation w.r.t. $\Sigma^*$) to the languages in $\mathcal{L}$. In other words, $\mathcal{B}\mathcal{L}$ is the smallest family of languages containing $\mathcal{L}$ and closed under Boolean operations: if $L_1, L_2 \in \mathcal{L}$ then 
$L_1 \cap L_2 \in \mathcal{B}\mathcal{L}$ and $L_1 \cup L_2 \in \mathcal{B}\mathcal{L}$ and $L_1^c, L_2^c \in  \mathcal{B}\mathcal{L}$.
Similarly, define the concatenation closure of $\mathcal{L}$ as the smallest family of languages containing $\mathcal{L}$ and closed under concatenation: if $L_1, L_2 \in \mathcal{L}$ then $L_1 L_2 \in \mathcal{M}\mathcal{L}$.

We begin with the class $\mathcal{E}$ of basic languages consisting of $\{a_1\}, \ldots \{a_k\}, \{\epsilon\}, \emptyset$. By alternately applying the operators $\mathcal{B}$ and $\mathcal{M}$ to $\mathcal{E}$ we can define the hierarchy 
$$\mathcal{E} \subseteq \mathcal{M}\mathcal{E} \subseteq \mathcal{B}\mathcal{M}\mathcal{E} \subseteq \mathcal{M}\mathcal{B}\mathcal{M}\mathcal{E} \subseteq \ldots.$$
%; in particular, we could start with $\mathcal{E} \subseteq \mathcal{M}\mathcal{E} \subseteq \mathcal{B}\mathcal{M}\mathcal{E} \subseteq \ldots$. While this does not make much difference for the higher levels, it can make a difference at 
Let $\mathcal{B}_0 = \mathcal{B}\mathcal{M}\mathcal{E}$. The dot-depth hierarchy is the sequence of families of languages 
$\mathcal{B}_0 \subseteq \mathcal{B}_1 \subseteq \ldots$ defined inductively by 
$\mathcal{B}_{n+1} = \mathcal{B}\mathcal{M}\mathcal{B}_n$
 for $n \geq 0$. 
It is known that all the inclusions in $\mathcal{B}_0 \subseteq \mathcal{B}_1 \subseteq \ldots$ are strict and is exemplified by the languages $\dn{n}$ (see \citet{Pin}). 
Minor variations in the definition exist in the literature; in particular, we could have applied the operator $\mathcal{B}$ first, but these have only minor effects on the overall concept and results.

\section{Expressiveness Results}\label{sec:expr_res}

We define a weaker version of counter automata which are restricted in a certain sense. 
Then, we show that Transformers are at least as powerful as such automata.

\begin{definition}[Simplified and Stateless counter machine] \label{def:qscm}
	We define a counter machine to be simplified and stateless if $u$ and $\delta$ have the following form,
	\begin{equation*}
	u : \Sigma \rightarrow \{ +m : m \in \mathbb{Z} \}^k, \\
	\end{equation*} 
	\begin{equation*}
	\delta : \Sigma  \rightarrow Q
	\end{equation*}
\end{definition}
This implies that the machine can have $k$ counters. The counters can be incremented or decremented by any values but it will only depend on the input symbol. Similarly, the state transition will also depend on the current input. A string $x \in {\Sigma^{*}}$ will be accepted if $\langle q_n, z(\vc_n) \rangle \in F$. We use $L_{RCL}$ to denote the class of languages recognized by such a counter machine. The above language is similar to $\Sigma$-restricted counter machine defined in \cite{merrill2020formal}.

\begin{lemma}\label{lem:qscm}
Transformers can recognize $L_{RCL}$.
%	The class of Transformers as defined in Section \ref{esec:def} is at least as powerful as the counter machine defined in \ref{def:qscm}.
\end{lemma}

\begin{proof}
	Let $s_1, s_2, \ldots, s_n$ denote a sequence $w \in \Sigma^*$. If the counter machine has $k$ counters, then let the dimension of intermediate vectors $d_{model}= 2k + |\Sigma|$. The first $2k$ dimensions will be reserved for counter related operations and  then $|Q|$ dimensions will be reserved to obtain the state vector. The embedding vector $\vx_i$ of each symbol will have $0$s in the first $2k$ dimensions and the last $|\Sigma|$ dimensions will have the one-hot encoding representation of the symbol. For a $k$ counter machine the value vectors would have a subvector of dimension 2 reserved for computations pertaining to each of the counter. That is, $\vx_{2j:2j+1}$ will be reserved for the $j$th counter where $0 \leq j < k$. For any given input symbol $s$, if $u(s)$ has counter operation of $+m$ at the $j$th counter, then the value will be such that $\vv$ will contain $+m$ at index $2j$ and $-m$ at index $2j+1$ upto index $2k$. The last $|\Sigma|$ dimensions will have the value $0$ in the value vectors. This can be easily obtained by a linear transformation $V(.)$ over one-hot encodings. The linear transformation $K(.)$ to obtain the key vectors will lead to zero vectors and hence all inputs will have equal attention weights. The linear transformation $V(.)$ to obtain the value vectors $\vv_i$ will be identity function. Hence the output of the self-attention block along with residual connection will be of the form $\va_i = \frac{1}{i}\sum_{t=1}^{i}\vv_t + \vx_i$. 
	
	The last $|\Sigma|$ dimensions of the vector $\va_i$ will have one-hot encoding of the input vector at $i$-th step. The one-hot encoding of the input can be easily mapped to the one-hot encoding for the corresponding state using a simple FFN. Additionally, this will ensure that, at the $i$-th step, the output of the self-attention block $\va_i$ will have the value $\frac{c_j}{i}$ at indices $2j$, where $c_j$ denotes the counter value of the counter automata representing the language. Similarly, the odd indices $2j+1$ will have the value $-\frac{c_j}{i}$. After applying a simple feed-forward network with $\textsf{ReLU}$ activation, we obtain the output vector $\vz_i$. It is easy to implement the zero check function with a simple linear layer over the output vector.  The network accepts an input sequence $w$ when the values in the output vector corresponding to each counter and state at the $n$-th correspond to that required for the final state. 
	
\end{proof}

We next show that $n$-ary Boolean Expressions can be recognized by Transformers with a similar construction. 
\begin{lemma}\label{lem:bool}
	Transformers can recognize $n$-ary Boolean Expressions. 

\end{lemma}
\begin{proof}
	Let $L_m$ denote a language of type $n$-ary Boolean Expressions with $m$ operators defined over the alphabet $\Sigma$. Consider a single layer Transformer network with $d_{model} = 2$. Let $s_0, s_1, \ldots, s_n$ be sequence $w$ where $w \in \Sigma^*$. Let $s_0$ be a special start symbol with embedding $f_e= [+1, -1]$. The embeddings of each input symbol $s \in \Sigma$ are defined as follows, $f_e(s) = [+(r-1), - (r-1) ]$ where $r$ denotes the arity of the symbol. The arity of values such as $0$ and $1$ is taken as $0$. Similar to the previous construction, the key values are null and hence attention weights are uniform leading to $\va_i = \frac{1}{i}\sum_{t=1}^{i}\vv_t$. Hence the output of the self-attention block will be $\va_i= [\frac{c_j}{i},-\frac{c_j}{i}]$, where $c_j$ denotes the counter value of the automata representing the language. Essentially, for each operator, the value added to the counter is equal to its arity subtracted by 1. For each value such as $0$ and $1$, the counter value is decremented by $1$. We then apply a simple FFN with $\textsf{ReLU}$ activation to obtain the output vector $\vz_i = \mathsf{ReLU}(\rmI\va_i)$.
	
	An input sequence $w$ belongs to the language $L_m$ if the second coordinate of the output is zero at every step, that is, $\vz_{i,2}=0$ for $0 \leq i \leq n$ and $\vz_n = \vzero$. 
\end{proof}

Let {\ResetDyck} be a language defined over alphabet $\Sigma=\{[,],1\}$, where $1$ denotes a symbol that requires a reset operation. Words in {\ResetDyck} have the form $\Sigma^*1v$, where the string $v$ belongs to Dyck-1. So essentially, when the machine encounters the reset symbol $1$, it has to ignore all the previous inputs, reset the counter to $0$ and go to start state. 

\begin{lemma}\label{lem:reset}
	A single-layer Transformer with only positional masking cannot recognize the language {\ResetDyck}.
\end{lemma}
\begin{proof}
	
	The proof is straightforward. Let $s_1, s_2, \ldots, s_n$ be an input sequence $w$. Let $s_r$ denote the $r$-th symbol where the reset symbol occurs. It is easy to see that the scoring function $\langle \vq_n, K(\vv_r) \rangle$ is independent of the position as well as the inputs before the reset symbol which are relevant for the reset operation. Consider the case where the first half of the input contains a sequence of open and closing brackets such that it does not belong to \dyck{} and the second half contains a sequence that belongs to \dyck{}. If the reset symbol occurs after the first half of the sequence, then the word belongs to \dyck{} and if it occurs in the beginning then it does not belong to the language \dyck{}. However, by construction, the output of the model $\vz_n$ will remain the same regardless of the position of the reset symbol and hence by contradiction, it cannot recognize such a language.
	
\end{proof}

The above limitation does not exist if there is a two layer network. The scoring function as well as value vector of the reset symbol will be dependent of the inputs that precede it. Hence it is not necessary that a two layer network will not be able to recognize such a language. Indeed, as shown in the main paper, the 2-layer Transformer performs well on {\ResetDyck}.

\begin{lemma}\label{lem:reg_pos}
	Transformers with only positional masking cannot recognize the language $(aa)^*$.
\end{lemma}
\begin{proof}
	
	Let $s_1, s_2, \ldots, s_n$ be an input sequence $w$ where $w \in a^*$. Since it is a unary language, the input at each step will be the same symbol and hence the embedding as well as query, key and value vectors will be the same. Since all the value vectors are the same, regardless of the attention weights, the output of the self-attention vector $\va_i$ will be a constant vector at each timestep. This implies that the output vectors $\vz_1 = \vz_2 = \ldots = \vz_n$. Inductively, it is easy to see that regardless of the number of layers this phenomenon will carry forward and hence the output vector at each timestep will be the same. Thus, the network cannot distinguish output at even steps and odd steps which is necessary to recognize the language $(aa)^*$.
	
\end{proof}

For parity, in the case where the input consists of only $1$s, the problem reduces to recognizing $(11)^*$. Hence it follows from the above result that a network without positional encoding cannot recognize parity even for minimal lengths.

\section{Experiments}\label{sec:exp_details}

\subsection{Discussion on Character Prediction Task}\label{sec:cpr}

As described in section \ref{sec:train_details}, we use character prediction task in our experiments to evaluate the model's ability to recognize a language. In character prediction task the model is only presented with positive samples from a given language and its goal is to predict the next set of valid characters. During inference, the model predicts the next set of legal characters at each step and a prediction is considered to be correct if and only if the model's output at every step is correct. The character prediction task is similar to predicting which of the input characters are allowed to make a transition in a given automaton such that it leads to a non-dead state. If an input character is not among the legal characters, that implies the underlying automaton will transition to a dead state and regardless of the following characters, the input word will never be accepted. When the end-of-sequence symbol is allowed as one of the next set of legal characters, it implies that the underlying automaton is in the final state and the input can be accepted.

\noindent \textbf{Character prediction and classification.} If a model can perform character prediction task perfectly, then it can also perform classification in the following way. For an input sequence $s_1, s_2, \ldots, s_n$, the model receives the sequence $s_1, \ldots, s_i$ for $1 \leq i \leq n$ at each step $i$ and  model predicts the set of valid characters in the $(i+1)^{th}$ position. If the next character is among the model's predicted set of valid characters at each step $i$ and the end of symbol character is allowed at the $n$-th step, then the word is accepted and if any character is not within the model's predicted set of valid characters, then the word is rejected. One of the primary reason for the choice of character prediction task is that it is arguably more robust than the standard classification task. The metric for character prediction task is relatively stringent and the model is required to model the underlying mechanism as opposed to just one label in standard classification. Note that the null accuracy (accuracy when all the predictions are replaced by a single label) is 50\% if the distribution of labels is balanced (higher otherwise), on the other hand the null accuracy of character prediction task is close to $0$. Additionally, in case of classification, depending on how the positive or negative data are generated, the model may also be biased to predict based on some statistical regularites instead of modeling the actual mechanism. In \cite{weiss2019learning}, they find that LSTMs trained to recognize Dyck-1 via classification on randomly sampled data do not learn the correct mechanism and fail on adversarially generated samples. On the other hand, \citet{suzgun2019lstm} show that LSTMs trained to recognize Dyck-1 via character prediction task learn to perform the correct mechanism required to do the task.

\noindent \textbf{Character prediction and language modelling.} The character prediction task has clear connections with Language modelling. If a model can perform language modelling perfectly, then it can perform character prediction task in the following way. For an input sequence $s_1, s_2, \ldots, s_n$, the model receives the sequence $s_1, \ldots, s_i$, for $1 \leq i \leq n$ at each step $i$ and predicts a distribution over the vocabulary. Mapping all the characters for which the model assigns a nonzero probability to $1$ and mapping to $0$ for all characters that are assigned zero probability will reduce it to character prediction task. However, there are a few issues with using language modelling in our formal language setting. Firstly, as mentioned in \cite{suzgun2019lstm}, the task of recognizing a language is not inherently probabilistic. Our goal here is to understand whether a network can or cannot model a particular language. Using language modelling will require us to 
impose a distribution arbitrarily for the given setting. More importantly, in character prediction task, some signals are explicitly provided. In the case of language modelling, we may just have to rely on the model to pick up those nuanced signals. For instance, in the language $\dn{n}$, when the input reaches the maximum depth $n$, in character prediction task it is explicitly provided the target value that $a$ is not allowed anymore whereas in language modelling the model is expected to assign zero probability to $a$ at the maximum depth based on the fact that it will never see a word depth more than $n$ in the training data. This phenomenon has major issues. For instance, when we consider Dyck-1 in practical setting, we can only provide it with limited data which implies there will be a sequence with a maximum finite depth. In this scenario, a language model trained on such data may learn the Dyck-1 language or the language $\dn{n}$ with that particular maximum depth. This limitation does not exist in the character prediction task where the signal is explicitly provided during training.

\subsection{Experimental Details}\label{sec:expdetails}

\begin{table*}[t]
	\small{\centering
		
		\begin{tabular}{P{9em}P{4em}P{4em}P{4em}P{4em}P{4em}P{4em}}
			\toprule
			&\multicolumn{2}{c}{\textbf{Training Data}} &\multicolumn{4}{c}{\textbf{Test Data}}\\
			\cmidrule(lr){2-3}\cmidrule(lr){4-7}
			Language & Size  & Length Range & Size per Bin & Length Range & Number of Bins & Bin Width \\
			\midrule
			\multicolumn{7}{c}{\textbf{Counter Languages}}\\
			\midrule
			Shuffle-2 & 10000& [2, 50] & 2000 & [2, 150] & 3 & 50\\
			Shuffle-4 & 10000& [2, 100] & 2000 & [2, 200] & 3 & 50\\
			Shuffle-6 & 10000& [2, 1000] & 2000 & [2, 200] & 3 & 50\\
			Boolean-3 & 10000 & [2, 50] & 2000 & [2, 150] & 3 & 50\\
			Boolean-5  & 10000 & [2, 50] & 2000 & [2, 150] & 3 & 50\\
			$a^nb^n$& 50 & [2, 100] & 50 & [2, 300] & 3 & 100\\
			$a^nb^nc^n$ & 50 & [3, 150] & 50 & [3, 450] & 3 & 150\\
			$a^nb^nc^nd^n$ & 50 & [4, 200] & 50 & [4, 600] & 3 & 200\\
			Dyck-1 & 10000& [2, 50] & 2000 & [2, 150] & 3 & 50\\
			\midrule
			\multicolumn{7}{c}{\textbf{Regular Languages}}\\
			\midrule
			Tomita 1 & 50 & [2, 50] & 100 & [2, 100] & 2 & 50\\
			Tomita 4 & 10000 & [2, 50] & 2000 & [2, 100] & 2 & 50\\
			Tomita 7 & 10000 & [2, 50] & 2000 & [2, 100] & 2 & 50\\
			Tomita 2 & 25 & [2, 50] & 50 & [2, 100] & 2 & 50\\
			$aa^*bb^*cc^*dd^*ee^*$ & 10000 & [5, 200] & 1000 & [5, 300] & 2 & 100\\
			$\{a,b\}^*d\{b, c\}^*$ & 10000 & [1, 50] & 2000	& [1, 100] & 2 & 50\\
			%$\{a,b\}^*d\{b, c\}^*d\{c,a\}^*$ & 10000 & [1, 50] & 2000	& [1, 100] & 2 & 50\\				
			$\{0,1,2\}^*02^*$ & 10000 & [2, 50] & 2000 & [2, 100] & 2 & 50\\
			%$c^*aaac^*$ & 1000 & [3, 50] & 500 & [3, 100] & 2 & 50\\
			%$c^*aaac^*bbbc^*$ & 1000 & [6, 50] & 500 & [6, 100] & 2 & 50\\
			%$c^*ac^*bc^*ac^*$ & 1000 & [3, 50] & 500 & [3, 100] & 2 & 50\\
			$\dn{2}$ & 10000& [2, 100] & 2000 & [2, 200] & 2 & 100\\
			$\dn{3}$ & 10000& [2, 100] & 2000 & [2, 200] & 2 & 100\\
			$\dn{4}$ & 10000& [2, 100] & 2000 & [2, 200] & 2 & 100\\
			$\dn{12}$ & 10000& [2, 100] & 2000 & [2, 200] & 2 & 100\\
			Parity & 10000& [2,50] & 2000 & [2, 100] & 2 & 50\\
			$(aa)^{*}$ & 250 & [2, 500] & 50 & [2, 600] & 2 & 100\\
			$(aaaa)^{*}$ & 125 & [4, 500] & 25 & [4, 600] & 2 & 100\\
			$(abab)^{*}$ & 125 & [4, 500] & 25 & [4, 600] & 2 & 100\\
			Tomita 3 & 10000 & [2, 50] & 2000 & [2, 100] & 2 & 50\\
			Tomita 5 & 10000 & [2, 50] & 2000 & [2, 100] & 2 & 50\\
			Tomita 6 & 10000 & [2, 50] & 2000 & [2, 100] & 2 & 50\\

			\bottomrule
		\end{tabular}
		\caption{\label{tab:stats}Statistics of different datasets used in the experiments. Note that the width of the first bin is always defined by the training set (see \ref{sec:exp_details}), and hence can be different from the widths of other bins reported in Bin Width column. As an example, for $(aa)^*$, the first bin will have a length range of [2, 500] and [502, 600] for the second bin.}

	}
\end{table*}

We use 4 NVIDIA Tesla P100 GPUs each with 16 GB memory to run our experiments, and train and evaluate our models on about 9 counter languages and 18 regular languages. The important details of all of these languages like the training and test sizes and the lengths of the strings considered, have been summarized in Table \ref{tab:stats}. In all of our experiments, the first bin always has the same length range as the training set, i.e. if the training set contains strings with lengths in range $[2, 50]$, then the strings in the first test bin will also lie in the same range. Width of bin is the difference between upper and lower limits of the string lengths that lie in that bin. All the test bins are taken to be disjoint from each other. Hence, if we have 3 bins with a width of 50 and the training range is $[2, 50]$, then the length ranges for the test bins will be $[2, 50]$, $[52, 100]$ and $[102, 150]$.

\begin{table*}[t]
		\normalsize{\centering
	\begin{tabular}{P{9em}P{15em}}
		\toprule
		\textbf{Hyperparameter} & \textbf{Bounds}\\
		\midrule
		Hidden Size& [3, 32]\\
		Heads & [1, 4] \\
		Number of Layers & [1, 2] | [1, 4]\\
		Learning Rate & [1e-2, 1e-3]\\
		Position Encoding Scheme & [Absolute, Relative, Positional Masking]\\
		\bottomrule
	\end{tabular}
	\caption{\label{tab:hyps} Different hyperparameters and the values considered for each of them. Note that certain parameters like Heads and Position Encoding Scheme are only relevant for Transformer based models and not for LSTMs. We considered upto 4 layers transformers in the cases where the training accuracies obtained were low for single and two layered networks and reported the results accordingly.}
}
\end{table*}

For each of these languages, we extensively tune on a bunch of different architectural and optimization related hyperparameters. Table \ref{tab:hyps} lists the hyperparameters considered in our experiments and the bounds for each of them. This corresponds to about 162 different configurations for tuning transformers (for a hidden size of 3, 4 heads are not allowed)  and 40 configurations for LSTMs . Over all the languages and hyperparameters there were a minimum of 117 parameters and a maximum of 17,888 parameters for the models that we considered. We use a grid search procedure to tune the hyperparameters. While reporting the accuracy scores for a given language, we compute the mean of the top 5 accuracies, corresponding to all hyperparameter configurations. For some experiments we had to consider the hyperparameters lying outside of the values specified in Table \ref{tab:hyps}. As an instance, we considered 4 layer transformers in the cases where the training accuracies obtained were low for single and two layered networks and reported the results accordingly. 

For training our models we used RMSProp optimizer with the smoothing constant $\alpha = 0.99$. In our initial few experiments we also tried Stochastic Gradient Descent with learning rate decay and Adam Optimizer, but decided to go ahead with RMSProp as it outperformed SGD in majority of experiments and gave similar performance as Adam but needed fewer hyperparameters. For each language we train models corresponding to each language for 100 epochs and a batch size of 32. In case of convergence, i.e. perfect accuracies for all the bins, before completion of all epochs, we stop the training process early. The results of our experiments on counter and regular languages are provided in Tables \ref{tab:countresults} and \ref{tab:regresults} respectively.

\begin{table*}[t]
	\scriptsize{\centering
		\begin{tabular}{P{12em}p{20em}P{7em}P{7em}P{7em}}
			\toprule
			\textbf{Language} & \textbf{Model} & \textbf{Bin-1 Accuracy [1, 50]}$\uparrow$ & \textbf{Bin-2 Accuracy [51, 100]}$\uparrow$ & \textbf{Bin-3 Accuracy [101, 150]}$\uparrow$ \\
			\midrule
			\multirow{5}{*}{\textbf{Dyck-1}} & \textbf{LSTM (Baseline)} & \textbf{100.0}  & \textbf{100.0} & \textbf{100.0} \\
			\cmidrule{2-5}
			&\textbf{Transformer (Absolute Positional Encodings)} & \textbf{100.0} & \textbf{100.0} & \textbf{100.0} \\
			&\textbf{Transformer (Relative Positional Encodings)} & \textbf{100.0} & 91.0 & 60.7 \\
			&\textbf{Transformer (Only Positional Masking)} & \textbf{100.0} & \textbf{100.0} & \textbf{100.0} \\
			\midrule
			%\toprule
			\multirow{5}{*}{\textbf{Shuffle-2}} & \textbf{LSTM (Baseline)} & \textbf{100.0}  & \textbf{100.0} & \textbf{100.0}\\
			\cmidrule{2-5}
			&\textbf{Transformer (Absolute Positional Encodings)} & \textbf{100.0} & 85.2 & 63.3 \\
			&\textbf{Transformer (Relative Positional Encodings)} & \textbf{100.0} & 51.6 & 3.8 \\
			&\textbf{Transformer (Only Positional Masking)} & \textbf{100.0} & \textbf{100.0} & 93.0 \\
			\midrule
			%\toprule
			\multirow{5}{*}{\textbf{Shuffle-4}} & \textbf{LSTM (Baseline)} & \textbf{100.0}  & \textbf{100.0} & \textbf{99.6} \\
			\cmidrule{2-5}
			&\textbf{Transformer (Absolute Positional Encodings)} & \textbf{100.0} & 46.6 & 20.8 \\
&\textbf{Transformer (Relative Positional Encodings)} & \textbf{100.0}& 57.2 & 5.5 \\
			&\textbf{Transformer (Only Positional Masking)} & \textbf{100.0} & \textbf{100.0} &98.8 \\
			\midrule
			\multirow{5}{*}{\textbf{Shuffle-6}} & \textbf{LSTM (Baseline)} & \textbf{100.0}  & \textbf{99.9} & 99.5 \\
			\cmidrule{2-5}
			&\textbf{Transformer (Absolute Positional Encodings)} & \textbf{100.0} &50.4 & 16.6 \\
			&\textbf{Transformer (Relative Positional Encodings)} & \textbf{100.0} & 59.1 & 5.7 \\
			&\textbf{Transformer (Only Positional Masking)} & \textbf{100.0} &\textbf{99.9} & 94.0 \\
			\midrule
			\multirow{5}{*}{\textbf{Boolean Expressions (3)}} & \textbf{LSTM (Baseline)} & \textbf{100.0}  & \textbf{100.0} & 99.7 \\
			\cmidrule{2-5}
			&\textbf{Transformer (Absolute Positional Encodings)} & \textbf{100.0} & 90.6 & 51.3 \\
			&\textbf{Transformer (Relative Positional Encodings)} & \textbf{100.0} & 96.0 & 68.4 \\
			&\textbf{Transformer (Only Positional Masking)} & \textbf{100.0} & \textbf{100.0} & \textbf{99.8} \\
			\midrule
			\multirow{5}{*}{\textbf{Boolean Expressions (5)}} & \textbf{LSTM (Baseline)} & \textbf{100.0}  &99.5 & 96.0 \\
			\cmidrule{2-5}
			&\textbf{Transformer (Absolute Positional Encodings)} & \textbf{100.0} & 84.3 & 40.8 \\
			&\textbf{Transformer (Relative Positional Encodings)} & \textbf{100.0} & 72.3 & 32.3 \\
			&\textbf{Transformer (Only Positional Masking)} & \textbf{100.0} & \textbf{99.8} & \textbf{99.0} \\
			\midrule
			\multirow{5}{*}{\textbf{$\Ctwo$}} & \textbf{LSTM (Baseline)} & \textbf{100.0}  & \textbf{100.0} & 99.9 \\
			\cmidrule{2-5}
			&\textbf{Transformer (Absolute Positional Encodings)} & \textbf{100.0} & \textbf{100.0} & \textbf{100.0} \\
			&\textbf{Transformer (Relative Positional Encodings)} & \textbf{100.0} & \textbf{100.0} & \textbf{100.0} \\
			&\textbf{Transformer (Only Positional Masking)} & \textbf{100.0} & \textbf{100.0} & \textbf{100.0} \\
			\midrule
			\multirow{5}{*}{\textbf{$\Cthree$}} & \textbf{LSTM (Baseline)} & \textbf{100.0}  & \textbf{100.0} & 97.8 \\
			\cmidrule{2-5}
			&\textbf{Transformer (Absolute Positional Encodings)} & \textbf{100.0} & 62.1 & 5.3 \\
			&\textbf{Transformer (Relative Positional Encodings)} & \textbf{100.0} & 31.3 & 22.0 \\
			&\textbf{Transformer (Only Positional Masking)} & \textbf{100.0} & \textbf{100.0} & \textbf{100.0} \\
			\midrule
			\multirow{5}{*}{\textbf{$a^nb^nc^nd^n$}} & \textbf{LSTM (Baseline)} & \textbf{100.0}  & \textbf{100.0} & 99.9 \\
			\cmidrule{2-5}
			&\textbf{Transformer (Absolute Positional Encodings)} &88.45 & 0.0  & 0.0 \\
			&\textbf{Transformer (Relative Positional Encodings)} & 41.1  & 0.0  & 0.0 \\
			&\textbf{Transformer (Only Positional Masking)} & \textbf{100.0} & \textbf{100.0} & 99.4 \\
			\midrule
		\end{tabular}
		\caption{\label{tab:countresults} The performance of Transformers and LSTMs on the respective counter languages. Refer to section \ref{sec:res_count} in the main paper for details.}
	}
\end{table*}

\begin{table*}[t]

	\small{\centering
	\begin{tabular}{P{9em}P{3em}P{3em}P{3em}P{3em}P{3em}P{3em}P{3em}P{3em}}
		\toprule
		&&&\multicolumn{2}{c}{\thead{ \textbf{Transformer} \\ \textbf{(Only Positional Masking)}}} &\multicolumn{2}{c}{\thead{ \textbf{Transformer} \\ \textbf{(w Position Encodings)}}}&\multicolumn{2}{c}{\thead{\textbf{LSTM}}} \\ 
		\cmidrule(lr){4-5}\cmidrule(lr){6-7}\cmidrule(lr){8-9}
	 	\textbf{Language} &\textbf{Property} &\textbf{Dot-Depth}& \textbf{Bin 0}& \textbf{Bin 1}& \textbf{Bin 0}& \textbf{Bin 1} & \textbf{Bin 0}& \textbf{Bin 1}\\
		\midrule
		Tomita 1& SF &1&100.0&100.0&100.0&100.0&100.0&100.0\\
		Tomita 4& SF (LT-$k$) &1&24.1&0.2&100.0&92.4&100.0&100.0\\
		Tomita 7&SF &1&100.0&100.0&99.9&99.8&100.0&100.0\\
		Tomita 2 = $\dn{1} = (01)^*$& SF &1&100.0&100.0&100.0&100.0&100.0&100.0\\
		$aa^*bb^*cc^*dd^*ee^*$&SF&1& 100.0 & 100.0 & 100.0 & 100.0 &100.0 & 100.0\\
		$\{a,b\}^*d\{b, c\}^*$ & SF&1& 100.0 & 100.0 & 100.0 & 100.0 &100.0 & 100.0\\
		%$\{a,b\}^*d\{b, c\}^*d\{c,a\}^*$&SF&1&\\
		$\{0,1,2\}^*02^*$ &SF&2&74.2&35.6&100.0&68.7&100.0&100.0\\
		$\dn{2}$&SF &2&7.8&0.4&74.6&3.1&100.0&100.0\\
		$\dn{3}$& SF &3&16.2&4.2&80.9&8.5&100.0&100.0\\
		$\dn{4}$& SF &4&36.9&15.6&90.2&3.3&100.0&100.0\\
		$\dn{12}$& SF &12&16.5&0.0&95.8&1.5&100.0&100.0\\			
		Parity & non-SF & $-$ & 22.0 & 0.0 & 68.7 &0.0 & 100.0 & 100.0\\
		$(aa)^{*}$&non-SF&$-$  &0.0&0.0&100.0&0.0&100.0&100.0\\
		$(aaaa)^{*}$&non-SF&$-$  &0.0&0.0&100.0&0.0&100.0&100.0\\
		$(abab)^{*}$&non-SF&$-$  &0.0&0.0&100.0&2.5&100.0&100.0\\
		Tomita 3& non-SF&$-$ &9.8&9.8&75.4&10.8&100.0&100.0\\
		Tomita 5& non-SF&$-$  &4.9&0.0&29.3&0.0&100.0&100.0\\
		Tomita 6& non-SF &$-$ &9.1&0.0&88.8&0.0&100.0&100.0\\
		\bottomrule

	\end{tabular}
\caption{\label{tab:regresults}Summary of results on Regular Languages. The languages are arranged in an increasing order of their complexities.}
}

\end{table*}

\section{Plots}\label{sec:plots}

We visualize different aspects of the trained models to understand how they achieve a particular task and if the learned behaviour resembles our constructions. Figure \ref{fig:values} shows the value vectors corresponding to the trained models on Shuffle-2 and Boolean-3 Language. We also visualize the attention weights corresponding to these two models in Figure \ref{fig:attnwts}. Similar to the self-attention output visualizations for Shuffle-2 and Boolean-3 in the main paper, we visualize these values for a model trained on Shuffle-4 in Figure \ref{fig:attn_out_shuff4} and again, find close correlations with the depth to length ratios of different types of brackets in the language.   Finally, in Figure \ref{fig:learned_aa}, we visualize a component of the learned position embeddings vectors and found a similar behaviour to $\cos(n\pi)$ agreeing with our hypothesis.

\begin{figure}[t]
	\centering
	\begin{subfigure}{.4\textwidth}
		\centering
		\includegraphics[scale = 0.35]{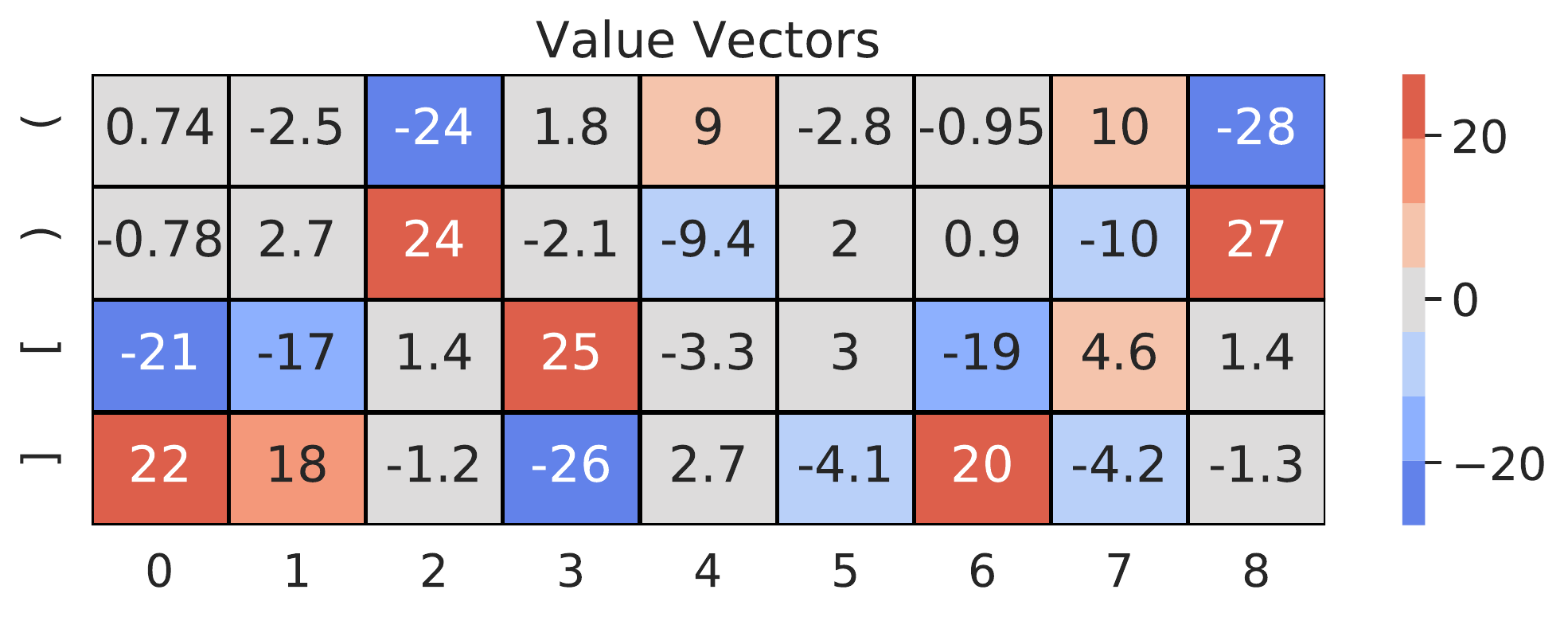}
		\caption{\label{fig:shuffvals}}
	\end{subfigure}
	\begin{subfigure}{.4\textwidth}
		\centering
		\includegraphics[scale=0.35]{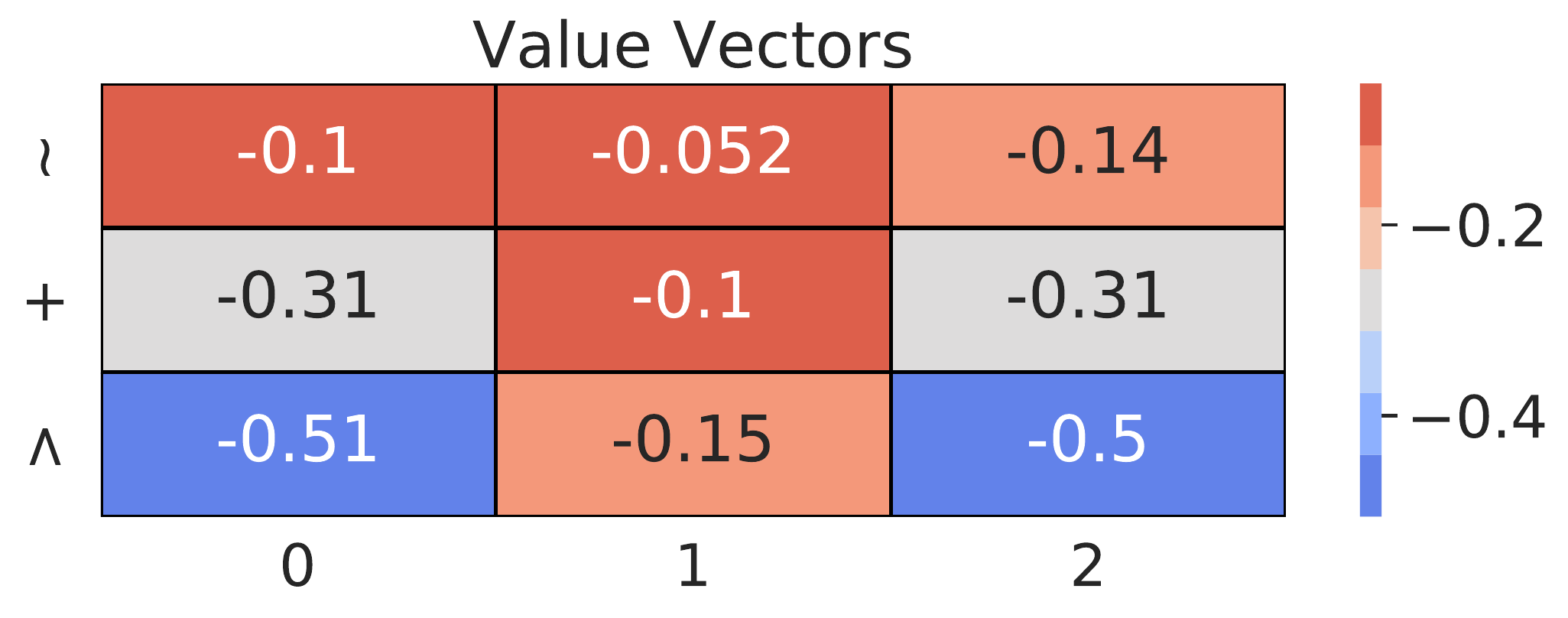}
		\caption{\label{fig:boolvals}}
	\end{subfigure}
	\caption{\label{fig:values} Plot of value vectors of  transformer based models trained on Shuffle-2 \ref{fig:shuffvals} and Boolean-3 language \ref{fig:boolvals}. The Shuffle-2 model had a hidden size of 8 and boolean-3 model had a hidden size of 3. The x-axis corresponds to different components of the value vectors for both models. Shuffle-2  language consisted of square and round brackets, while for Boolean-3 we considered 3 operators namely: $\sim$ a unary operator, $+$ a binary operator and finally, $>$ which is a ternary operator..}
\end{figure}

\begin{figure*}
	\centering
	\begin{subfigure}{.5\textwidth}
		\centering
		\includegraphics[scale = 0.12]{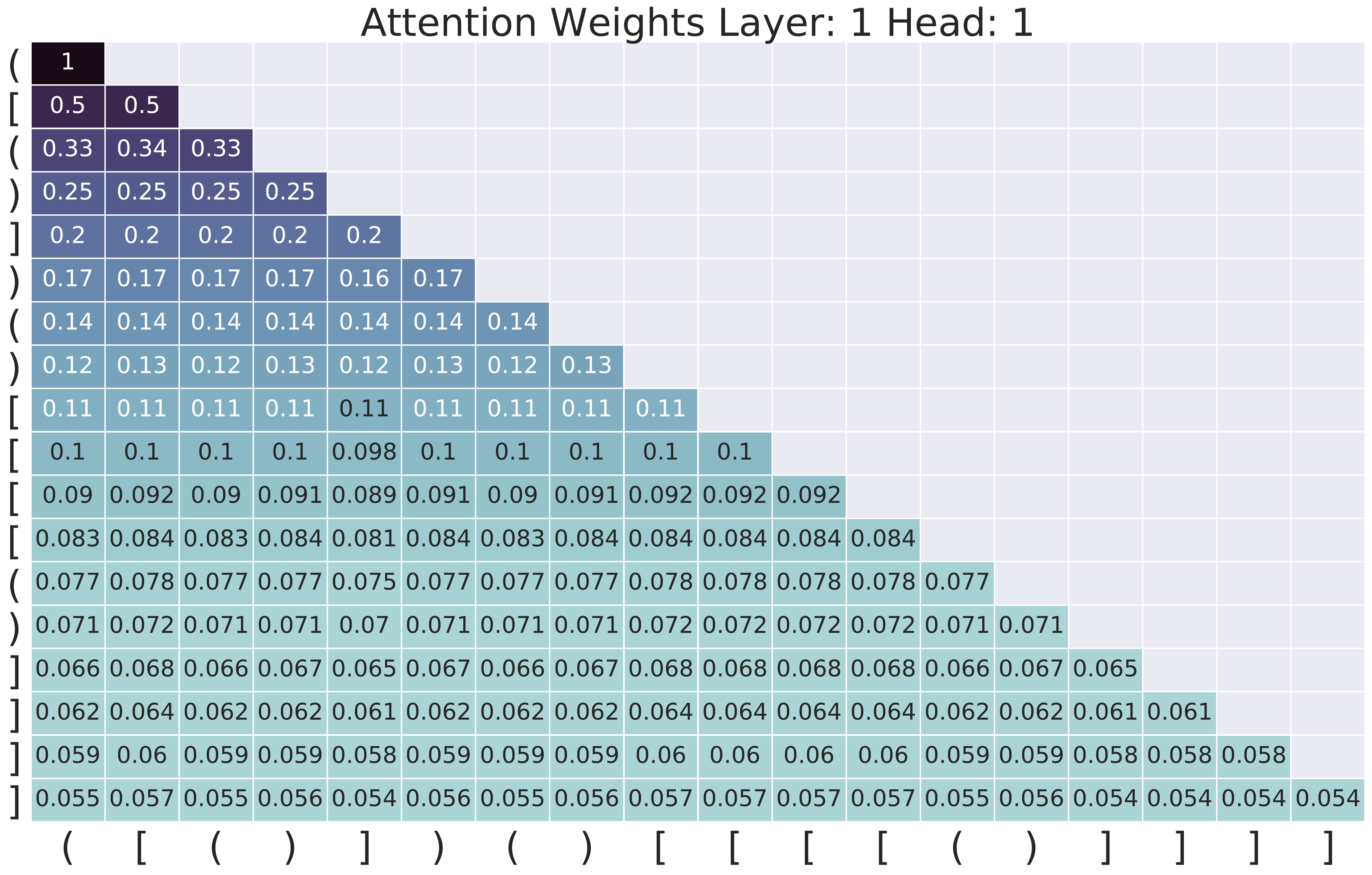}
		\caption{\label{fig:shuffattn}}
	\end{subfigure}%
	\begin{subfigure}{.5\textwidth}
		\centering
		\includegraphics[scale = 0.12]{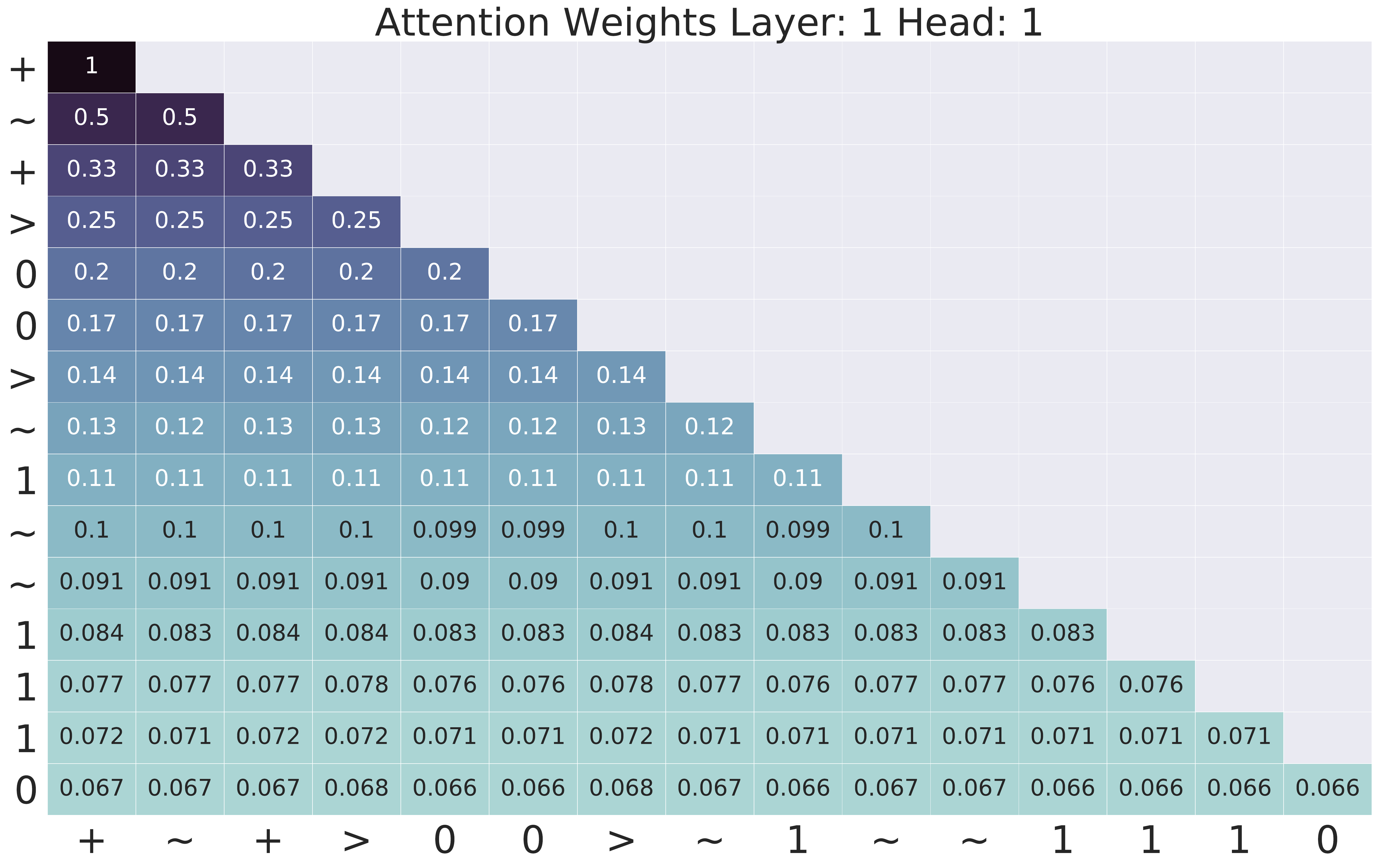}
		\caption{\label{fig:boolattn}}
	\end{subfigure}
	\caption{\label{fig:attnwts} Attention maps for models trained on Shuffle-2 and Boolean-3 languages. Similar to our constructions for recognizing these languages, we observe nearly uniform attention weights in both cases}
\end{figure*}

\begin{figure*}[t]
	\centering
	\begin{subfigure}{.5\textwidth}
		\centering
		\includegraphics[scale=0.12]{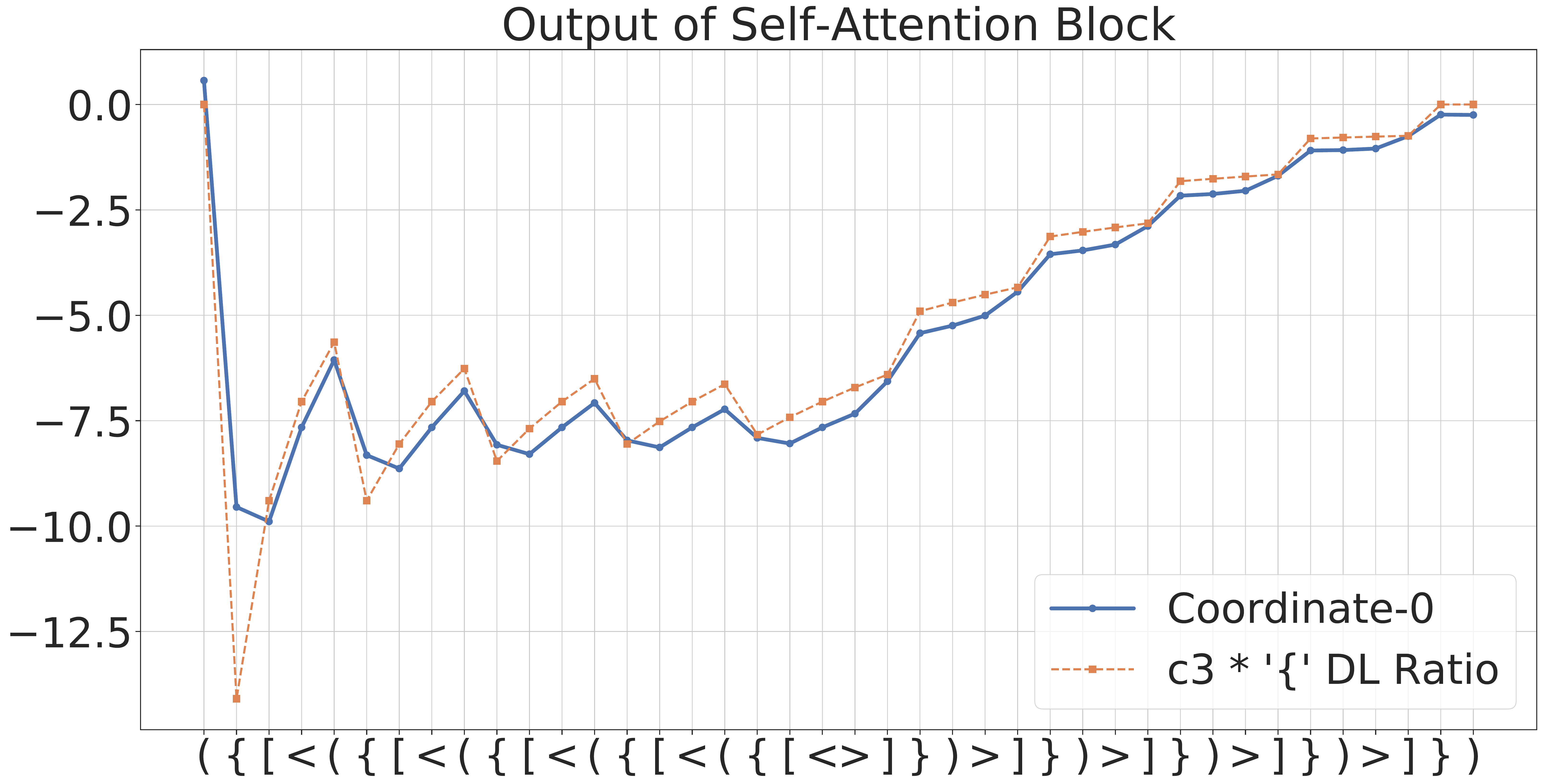}
		\caption{}
	\end{subfigure}%
	\begin{subfigure}{.5\textwidth}
		\centering
		\includegraphics[ scale = 0.12]{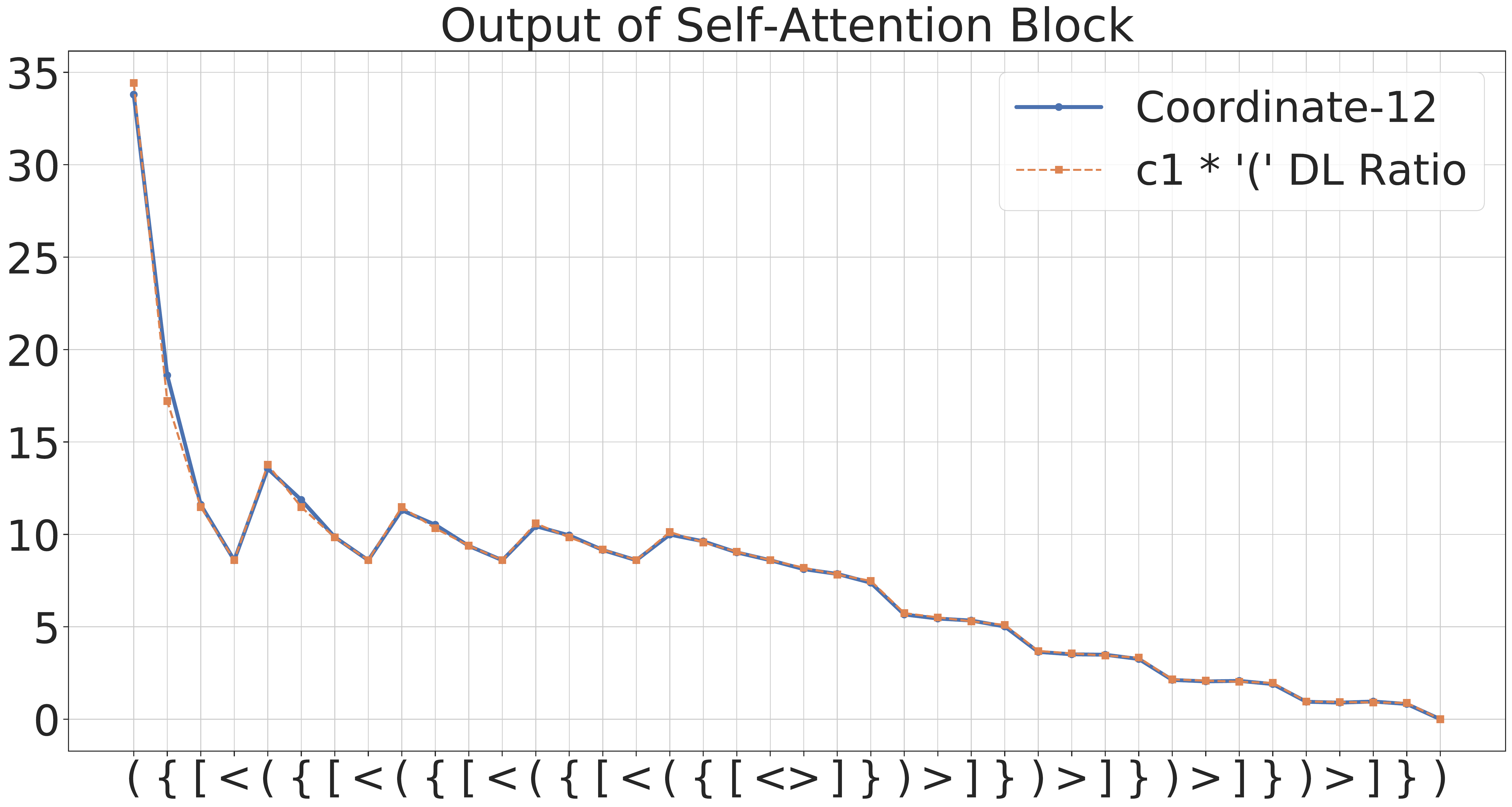}
		\caption{ }
	\end{subfigure}
	\newline	
	\begin{subfigure}{.5\textwidth}
		\centering
		\includegraphics[scale=0.12]{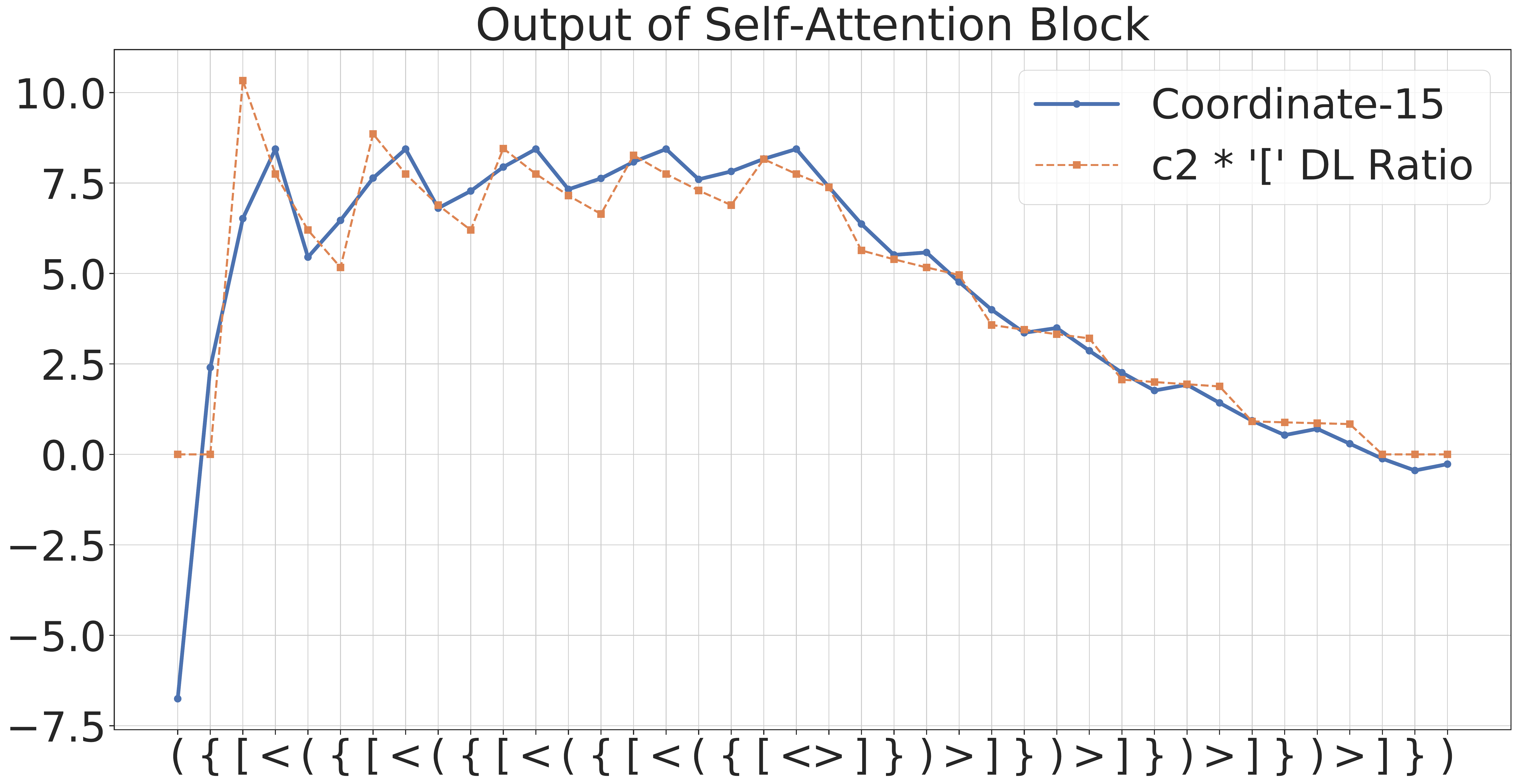}
		\caption{ }
	\end{subfigure}%
	\begin{subfigure}{.5\textwidth}
		\centering
		\includegraphics[scale=0.12]{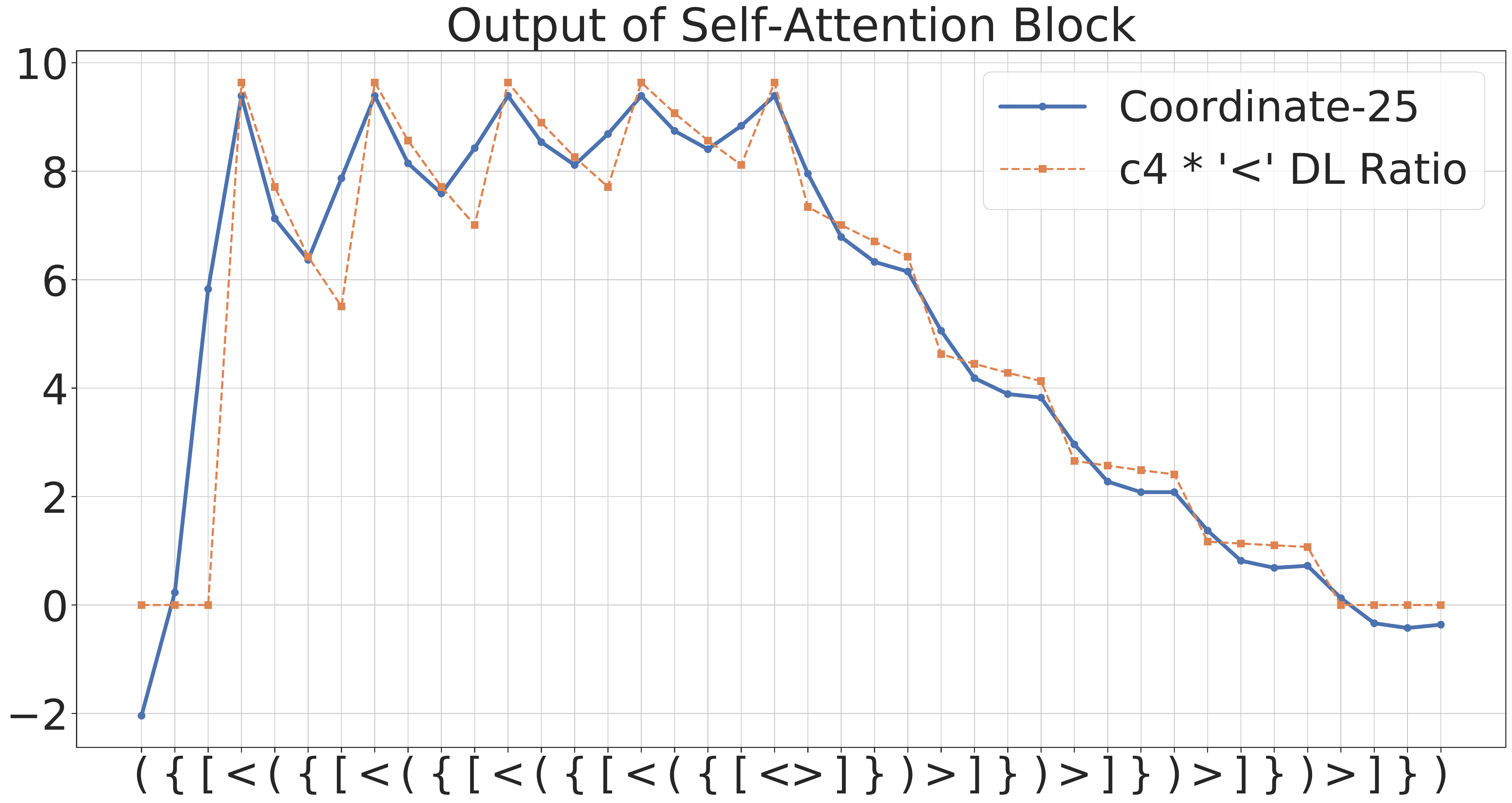}
		\caption{ }
	\end{subfigure}
	\caption{\label{fig:attn_out_shuff4} Values of four different coordinates of the output of self-attention block. The model is trained to recognize Shuffle-4. The dotted lines are the scaled depth to length ratio for the four types of bracket provided for reference. }
\end{figure*}

\begin{figure}
	
	\centering
	\includegraphics[scale=0.12]{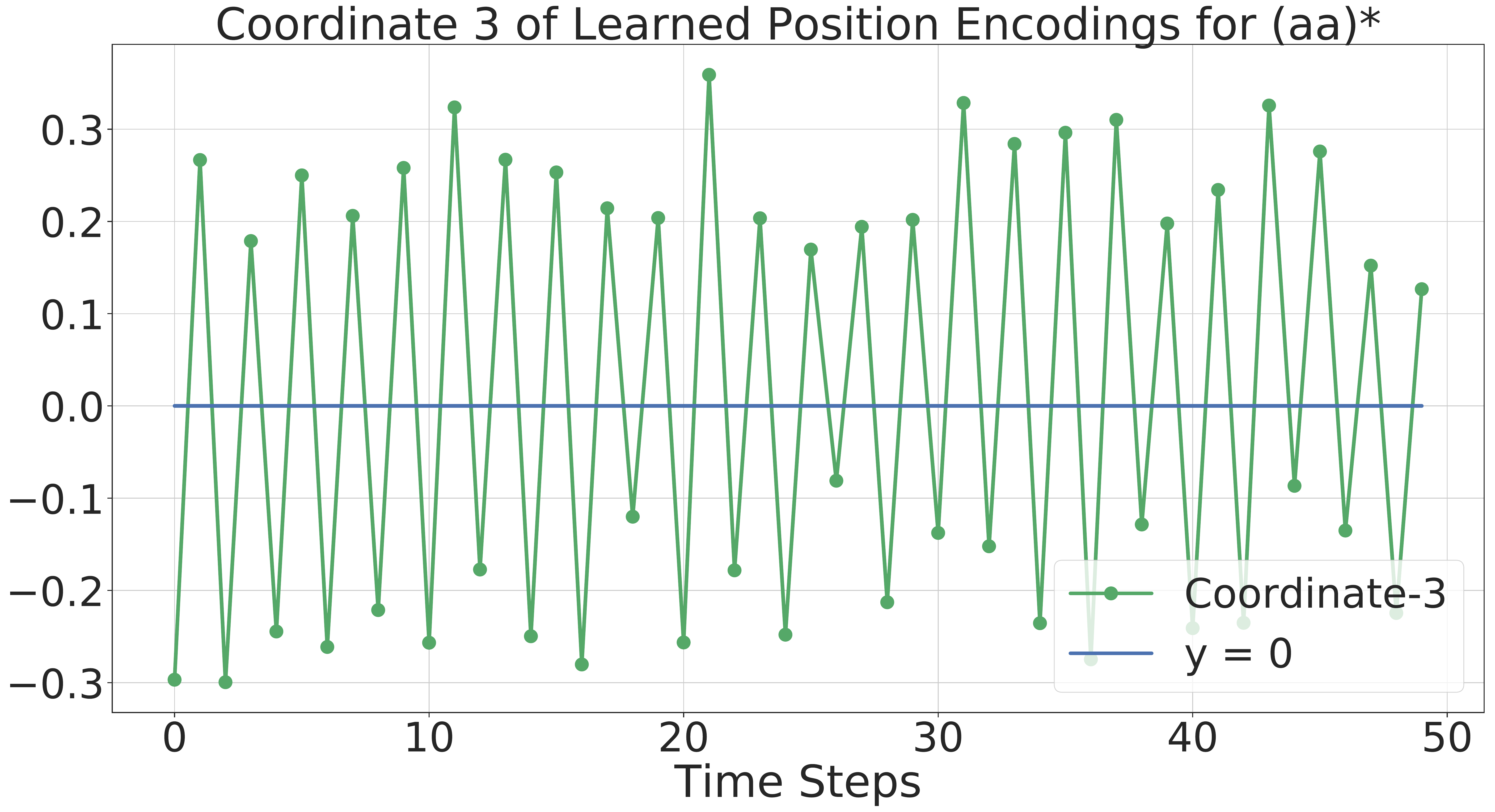}
	\caption{\label{fig:learned_aa} The values of coordiante 3 of the learned position encodings on the language $(aa)^*$. The variation in the encodings resemble a periodic behaviour similar to $cos(n\pi)$}
	
\end{figure}

\end{document}